\documentclass{article}

\usepackage{microtype}
\usepackage{graphicx}
\usepackage{subfigure}
\usepackage{booktabs} 

\usepackage{amsmath}
\usepackage{amssymb}
\usepackage{amsthm}
\usepackage{bbm}
\usepackage{xcolor}
\usepackage[shortlabels]{enumitem}

\usepackage{color}

\usepackage{thmtools}
\usepackage{thm-restate}

\declaretheorem[name=Theorem, numberwithin=section]{theorem}

\declaretheorem[name=Lemma, numberwithin=section]{lemma}
\declaretheorem[name=Corollary, numberwithin=section]{corollary}
\declaretheorem[name=Proposition, numberwithin=section]{proposition}

\declaretheorem[name=Remark, numberwithin=section]{remark}

\newtheorem{algorithm}{Algorithm}
\newtheorem{assumption}{Assumption}

\newlist{assA}{enumerate}{10}
\setlist[assA]{label={(\bfseries A\arabic*):}, ref={A\arabic*}, resume}

\newlist{assB}{enumerate}{10}
\setlist[assB]{label={(\bfseries B\arabic*):}, ref={B\arabic*}}

\newlist{assC}{enumerate}{10}
\setlist[assC]{label={(\bfseries C\arabic*):}, ref={C\arabic*}}

\DeclareMathOperator*{\argmin}{argmin}

\newcommand{\grad}{\nabla}
\newcommand{\jac}{\partial}
\newcommand{\hess}{\grad^2}
\newcommand{\norm}[1]{\lVert{#1}\rVert}

\newcommand{\fo}{E} 
\newcommand{\fin}{\ell} 

\newcommand{\R}{\mathbb{R}}
\newcommand{\N}{\mathbb{N}}

\newcommand{\Low}{\eta_{1,\lambda}}
\newcommand{\Lol}{\eta_{2,\lambda}}
\newcommand{\Bo}{L_{E,\lambda}}

\newcommand{\D}{D_\lambda}
\newcommand{\IMD}{AID}
\newcommand{\ITD}{ITD}

\newcommand{\q}{q_\lambda}

\newcommand{\rhol}{\nu_{2, \lambda}}
\newcommand{\rhow}{\nu_{1, \lambda}}
\newcommand{\LPhi}{L_{\Phi, \lambda}}

\newcommand{\rf}{\rho_\lambda}
\newcommand{\hrf}{\sigma_{\lambda}}

\newcommand{\Llw}{L_\fin(\lambda)}

\newcommand{\cond}{\kappa(\lambda)}
\newcommand{\rholl}{\hat\rho_{2, \lambda}}
\newcommand{\rhowl}{\hat\rho_{1, \lambda}}

\newcommand{\mufi}{\mu_\fin(\lambda)}

\newcommand{\sz}{\alpha(\lambda)}
\newcommand{\dsz}{\grad \alpha(\lambda)}

\newcommand{\pd}{\mu_\lambda}

\newcommand{\A}{A_\lambda}
\newcommand{\B}{b_\lambda}
\newcommand{\C}{c_\lambda}
\newcommand{\Al}{Z_\lambda}
\newcommand{\hA}{A_{\lambda,t}}
\newcommand{\hAl}{Z_{\lambda,t}}
\newcommand{\hB}{b_{\lambda,t}}
\newcommand{\hC}{c_{\lambda,t}}

\newcommand{\Z}{v(\lambda)}
\newcommand{\hZ}{v_t(\lambda)}

\newcommand{\wt}{w_t(\lambda)}

\newcommand{\vo}{v_{t,0}(\lambda)}
\newcommand{\vk}{v_{t,k}(\lambda)}
\newcommand{\vkm}{v_{t,k-1}(\lambda)}

\newcommand{\wpo}{u_{t,0}(\lambda)}
\newcommand{\wpk}{u_{t,k}(\lambda)}
\newcommand{\wpkm}{u_{t,k-1}(\lambda)}
\newcommand{\wpkmm}{u_{t,k-2}(\lambda)}

\newcommand{\Dt}{D}
\newcommand{\Dv}{D'}

\newcommand{\kert}{K}
\newcommand{\kerv}{K'}
\newcommand{\Xv}{X'}
\newcommand{\Xt}{X}
\newcommand{\yt}{y}
\newcommand{\yv}{y'}

\usepackage{hyperref}
\usepackage{cleveref}

\crefname{appsec}{appendix}{appendices}

\crefname{assAi}{assumption}{assumptions}
\crefname{assBi}{assumption}{assumptions}
\crefname{assCi}{assumption}{assumptions}

\usepackage[accepted]{icml2020}

\begin{document}

\twocolumn[
\icmltitle{On the Iteration Complexity of Hypergradient Computation}



\icmlsetsymbol{equal}{*}

\begin{icmlauthorlist}
\icmlauthor{Riccardo Grazzi}{iit,ucl}
\icmlauthor{Luca Franceschi}{iit,ucl}
\icmlauthor{Massimiliano Pontil}{iit,ucl}
\icmlauthor{Saverio Salzo}{iit}
\end{icmlauthorlist}

\icmlaffiliation{iit}{Computational Statistics and Machine Learning, Istituto Italiano di Tecnologia, Genoa, Italy}
\icmlaffiliation{ucl}{Department of Computer Science, University College London, London, UK}

\icmlcorrespondingauthor{Riccardo Grazzi}{riccardo.grazzi@iit.it}

\icmlkeywords{Machine Learning, ICML, Meta-learning, Hyperparameter Optimization, Gradient-based, bilevel, hypergradient, equilibrium}

\vskip 0.3in
]



\printAffiliationsAndNotice{}  

\begin{abstract}
We study a general class of bilevel problems,  consisting in the minimization of an upper-level objective which depends on the solution to a parametric fixed-point equation. 
Important instances arising in machine learning include hyperparameter optimization, meta-learning, and certain graph and recurrent neural networks. Typically the gradient of the upper-level objective (hypergradient) is hard or even impossible to compute exactly, which has raised the interest in approximation methods. We investigate some popular approaches to compute the hypergradient, based on reverse mode iterative differentiation and approximate implicit differentiation. Under the hypothesis that the fixed point equation is defined by a contraction mapping, we present a unified analysis which allows for the first time to quantitatively compare these methods, providing explicit bounds for their iteration complexity. This analysis suggests a hierarchy in terms of computational efficiency among the above methods, with approximate implicit differentiation based on conjugate gradient performing best. We present an extensive experimental comparison among the methods which confirm the theoretical findings. 
\end{abstract}


\section{Introduction}
Several problems arising in machine learning and related disciplines can be formulated as bilevel problems, where the lower-level problem is a fixed point equation whose solution is part of an upper-level objective.
Instances of this framework include hyperparameter optimization~\citep{maclaurin2015gradient,franceschi2017forward,liu2018darts,lorraine2019optimizing,elsken2019neural}, meta-learning \citep{andrychowicz2016learning,finn2017model,franceschi2018bilevel}, 
as well as recurrent and graph neural networks
\citep{almeida1987learning,pineda1987generalization,scarselli2008graph}.

In large scale scenarios, there are thousands or even millions of parameters to find in the upper-level problem, making black-box approaches like grid and random search \citep{bergstra2012random} or Bayesian optimization \citep{snoek2012practical} impractical. This has made gradient-based methods \citep{domke2012generic,maclaurin2015gradient,pedregosa2016hyperparameter} popular in such settings, but also it has raised the issue of designing efficient procedures to approximate the gradient of the upper-level objective (hypergradient) when finding a solution to the lower-level problem is costly.

The principal goal of this paper is to study the degree of approximation 
to the hypergradient of certain iterative schemes based on iterative or implicit differentiation\footnote{ 
The reader interested in the convergence analysis of gradient-based algorithms for bilevel optimization is referred to \cite{pedregosa2016hyperparameter, rajeswaran2019meta} and references therein.}.
In the rest of the introduction we present the bilevel framework, alongside some relevant examples in machine learning. We then outline the gradient approximation methods that we analyse in the paper and highlight our main contributions. Finally, we discuss and compare  our results with previous work in the field.

\paragraph{The bilevel framework.}
In this work, we consider the following bilevel  problem.
\begin{equation}
\begin{aligned}
\label{mainprob}
&\min_{\lambda \in \Lambda} f(\lambda) := \fo(w(\lambda), \lambda)\\
&\text{\ \ \ subject~to ~}w(\lambda) = \Phi(w(\lambda), \lambda),
\end{aligned}
\end{equation}
where $\Lambda$  is a closed convex subset of $\R^n$ and 
$\fo\colon \R^d\times \Lambda \to \R$ and $\Phi\colon\R^d \times \Lambda \rightarrow \R^d$
are continuously differentiable functions.  
We assume that 
the lower-level problem in \eqref{mainprob} (which is a fixed point-equation) admits a unique solution. However, in general, explicitly computing such solution is either impossible or expensive.
When $f$ is differentiable,  
this issue affects the evaluation of the hypergradient $\grad f (\lambda)$, which at best can only be approximately computed.

A prototypical example of the bilevel problem~\eqref{mainprob} is
\begin{equation}
\begin{aligned}
\label{mainprob1}
&\min_{\lambda \in \Lambda} f(\lambda) := \fo(w(\lambda), \lambda ))\\
&\text{\ \ \ subject~to ~}w(\lambda) =  {\rm arg}\min_{u\in \R^d} \fin(u, \lambda),
\end{aligned}
\end{equation}
where $\ell\colon \R^d\times \Lambda \to \R$
is a loss function, twice continuously differentiable and strictly convex w.r.t.~the first variable. Indeed if we let
$\Phi$ be such that
 $\Phi(w, \lambda) = w - \alpha(\lambda) \grad_1 \fin(w, \lambda)$, where $\alpha\colon\Lambda \to \R_{++}$ is differentiable, then problem~\eqref{mainprob1} and 
 problem~\eqref{mainprob} are equivalent. 
Specific examples of problem~\eqref{mainprob1}, which include hyperparameter optimization and meta-learning, are discussed in Section~\ref{sec:4.1}.


Other instances of the bilevel problem~\eqref{mainprob}, which are not of the special form of problem~\eqref{mainprob1}, arise in the context of so-called equilibrium models
(EQM). Notably, these comprise some types of connectionist models employed in domains with structured data. Stable recurrent neural networks \citep{miller2019stable}, graph neural networks \citep{scarselli2008graph} and the formulations by \citet{bai2019deep} belong to this class. EQM differ from standard (deep) neural networks in that the internal representations 
are given by fixed points of learnable dynamics rather than compositions of a finite number of layers.
The learning problem for such type of models can be written as 
\begin{equation}
\label{mainprob2}
\begin{aligned}
&\min_{\lambda=(\gamma, \theta)\in\Lambda} f(\lambda):= \sum_{i=1}^n E_i(w_i(\gamma), \theta),\\
&\text{\ \ \ subject~to }w_i(\gamma) = \phi_i(w_i(\gamma), \gamma), \text{ for } 1 \leq i \leq n,
\end{aligned}
\end{equation}
where the operators $\phi_i: \R^d \times \Lambda \rightarrow \R^d$ (here $\Phi=(\phi_i)_{i=1}^n$) are associated to the training points $x_i$'s, and the error functions $E_i$ are the losses incurred by a standard supervised algorithm on the transformed dataset $\{w_i(\gamma),y_i\}_{i=1}^n$. 
A specific example is discussed in Section~\ref{sec:4.2}

In this paper, we present a unified analysis which allows for the first time to quantitatively compare
popular methods to approximate $\grad f(\lambda)$ in the general setting of problem~\eqref{mainprob}. The strategies we consider can be divided in two categories:
\begin{enumerate}
\vspace{-.2truecm}
\item  \emph{Iterative Differentiation \rm{(\ITD)}} \citep{maclaurin2015gradient, franceschi2017forward, franceschi2018bilevel, finn2017model}. 
One defines the sequence of functions $f_t(\lambda) = \fo(w_t(\lambda), \lambda)$, where $w_t(\lambda)$ are 
the fixed-point iterates generated by the map $\Phi(\cdot,\lambda)$.
Then
 $\grad f(\lambda)$ is approximated by $\grad f_t(\lambda)$,   which in turn is computed using forward (FMAD) or reverse (RMAD) mode automatic differentiation \cite{griewank2008evaluating}. 
\vspace{-.2truecm}
\item \emph{Approximate Implicit Differentiation {\rm(\IMD)}} \citep{pedregosa2016hyperparameter, rajeswaran2019meta, lorraine2019optimizing}. 
First, an (implicit) equation for $\grad f(\lambda)$ is obtained through the implicit function theorem. Then, this equation is  
approximately solved by using a two stage scheme. We analyse two specific methods in this class: 
the \emph{fixed-point method} \citep{lorraine2019optimizing}, also referred to as recurrent backpropagation in the context of recurrent neural networks \cite{almeida1987learning,pineda1987generalization},
and the \emph{conjugate gradient method}~\cite{pedregosa2016hyperparameter}.
\end{enumerate}
Both schemes can be efficiently implemented using automatic differentiation \citep{griewank2008evaluating,baydin2018automatic} achieving similar cost in time, while \ITD{} has usually a larger memory cost than \IMD{}\footnote{This is true when ITD is implemented using RMAD, which is the standard approach 
when $\lambda$  is high dimensional. 
}.

\paragraph{Contributions.}
Although there is a vast amount of literature on 
the two hypergradient approximation
strategies previously described, it remains unclear whether it is better to use one or the other. In this work, we shed some light over this issue both theoretically and experimentally. Specifically our contributions are the following:
\begin{itemize}
\vspace{-.2truecm}
\item We provide iteration complexity results for \ITD\ and \IMD\ when the mapping defining the fixed point equation is  a contraction. In particular, we prove non-asymptotic linear rates for 
the approximation errors of both approaches.
\vspace{-.1truecm}
\item  
We make a theoretical and numerical comparison among different \ITD\ and \IMD\ strategies 
considering several experimental scenarios.
\vspace{-.2truecm}
\end{itemize}
We note that, to the best of our knowledge, non-asymptoptic rates of convergence for \IMD\ were only recently given in the case of meta-learning \cite{rajeswaran2019meta}.
Furthermore, we are not aware of any previous results
about non-asymptotic rates of convergence for \ITD. 

\paragraph{Related Work.}
Iterative differentiation for functions defined implicitly has been extensively studied in the automatic differentiation literature. In particular \citep[Chap.~15]{griewank2008evaluating} derives asymptotic linear rates for \ITD{} under the assumption that $\Phi(\cdot, \lambda)$ is a contraction.
Another attempt to  theoretically analyse 
 ITD is made by \citet{franceschi2018bilevel} in the context of the bilevel problem~\eqref{mainprob1}.
There,
the authors provide sufficient conditions for the asymptotic convergence of the set of minimizers of the approximate problem to the set of minimizers of the bilevel problem. In contrast, in this work, we give rates for the convergence of the approximate hypergradient $\grad f_t(\lambda)$ to the true one (i.e. $\nabla f(\lambda)$).
\ITD\ is also considered in \cite{shaban2019truncated} where $\grad f_t(\lambda)$ is approximated via a procedure which is reminiscent of truncated backpropagation. The authors bound the norm of the difference between $\grad f_t(\lambda)$ and its truncated version as a function of the truncation steps.  This is different from our analysis which directly considers the problem of estimating the gradient of $f$.

In the case of AID, an 
asymptotic analysis is presented 
in \cite{pedregosa2016hyperparameter},
where the author proves the convergence of an inexact gradient projection algorithm for the minimization of the function $f$ defined in problem~\eqref{mainprob1}, using
increasingly accurate estimates of $\grad f(\lambda)$. Whereas
 \citet{rajeswaran2019meta} present complexity results 
in the setting of meta-learning with biased regularization. 
Here, we extend this line of work by providing complexity results for \IMD\ in the more general setting of problem~\eqref{mainprob}.




We also mention the papers by \citet{amos2017optnet} and \citet{amos2019differentiable}, which present techniques to differentiate through the solutions of quadratic and cone programs respectively.
Using such techniques allows one to treat these optimization problems as layers of a neural network and to use backpropagation for the end-to-end training of the resulting learning model.
In the former work, the gradient is obtained by implicitly differentiating through the KKT conditions of the lower-level problem, while the latter performs implicit differentiation on the residual map of Minty's parametrization. 

A different approach to solve bilevel problems of the form \eqref{mainprob1} is presented by \citet{mehra2019penalty}, who consider a sequence of ``single level'' objectives involving a quadratic regularization term penalizing violations of the lower-level first-order stationary conditions. The authors provide asymptotic convergence guarantees for the method, as the regularization parameter tends to infinity, and show that it outperforms both \ITD\ and \IMD\ on different settings where the lower-level problem is non-convex.

All previously mentioned works except \cite{griewank2008evaluating} consider bilevel problems of the form \eqref{mainprob1}. Another exception is  \cite{liao2018reviving}, which proposes two improvements to recurrent backpropagation, one based on conjugate gradient on the normal equations
, and another based on Neumann series approximation of the inverse. 


\section{Theoretical Analysis}\label{sec:theory}
In this section we establish non-asymptotic bounds on the hypergradient (i.e. $\grad f(\lambda)$) approximation errors 
for both \ITD\ and \IMD\ schemes (proofs can be found in \Cref{sec:proofs}). 
In particular, in Section~\ref{se:dynaDiff} we address the iteration complexity of \ITD,
while in Section~\ref{se:AID}, after giving a general bound for \IMD,\ we
focus on two popular implementations of the \IMD\ scheme: one based on the conjugate gradient method and the other on the fixed-point method.

\vspace{-0.7mm}
\paragraph{Notations.}
We denote 
by $\norm{\cdot}$ applied to a vector (matrix) the Euclidean (spectral) norm. 
For a differentiable function 
$f\colon\R^n\to\R^m$
we denote by $f^\prime(x) \in \R^{m\times n}$
the derivative of $f$ at $x$. When $m=1$, we denote by
$\nabla f\colon \R^n \to \R^n$
the gradient of $f$.
For a real-valued function $g\colon \R^n\times \R^m \to \R$ we denote by $\nabla_1 g(x,y) \in \R^n$ and $\nabla_2 g(x,y)\in \R^m$ the partial derivatives w.r.t.~the first and second variable respectively. We also denote by $\nabla^2_{1} g(x,y) \in \R^{n\times n}$ and $\nabla^2_{12} g(x,y)\in \R^{n\times m}$ the second derivative of $g$ w.r.t.~the first variable and the mixed second derivative of $g$ w.r.t.~the first and second variable.
For a vector-valued function $h\colon\R^n \times\R^m \to \R^k$ 
we denote, by $\partial_1 h(x,y) \in \R^{k\times n}$ and $\partial_2 h(x,y) \in \R^{k\times m}$ the partial Jacobians  w.r.t.~the first and second variable respectively at $(x,y) \in \R^m \times \R^n$.

In the rest of the section, referring to problem~\eqref{mainprob}, we will group the assumptions as follows. Assumption~\ref{assA} is general while Assumption~\ref{ass:contraction} and \ref{assC} are specific enrichments for \ITD{} and \IMD{} respectively.

\begin{assumption}
\label{assA}
For every $\lambda \in \Lambda$,
\vspace{-1.6ex}
\setlist[enumerate]{itemsep=0mm}
\begin{enumerate}[{\rm(i)}]

\item\label{ass:fixedpoint} 
$w(\lambda)$ is the unique fixed point of $\Phi(\cdot, \lambda)$.

\item\label{ass:invert}
$I -\jac_1 \Phi(w(\lambda), \lambda)$ is invertible.

\item\label{ass:philip} 
$\jac_1\Phi (\cdot, \lambda)$ and $\jac_2\Phi (\cdot, \lambda)$ are Lipschitz 
continuous with constants $\rhow$ and $\rhol$ respectively.

\item\label{ass:foLip}  
$\grad_1 \fo(\cdot, \lambda)$ and $\grad_2 \fo(\cdot, \lambda)$
are Lipschitz continuous with constants $\Low$ and $\Lol$ respectively.
\end{enumerate}
\end{assumption}

A direct consequence of Assumption~\ref{assA}\ref{ass:fixedpoint}-\ref{ass:invert} and of the implicit function theorem is that $ w(\cdot)$ and $f(\cdot)$ are differentiable 
on $\Lambda$. Specifically, for every $\lambda \in \Lambda$, it holds that
\begin{align}
\label{eq:gradexist_1}
w'(\lambda) &= (I - \jac_1\Phi(w(\lambda), \lambda))^{-1}\jac_2 \Phi(w(\lambda), \lambda) \\
\label{eq:gradexist_2}\grad{f}(\lambda) &= \grad_2 \fo (w(\lambda), \lambda)  + w'(\lambda)^\top \grad_1 \fo (w(\lambda), \lambda).
\end{align}
See Theorem~\ref{th:gradexist} for details.
In the special case of problem~\eqref{mainprob1}, equation \eqref{eq:gradexist_1} reduces (see Corollary~\ref{cor:gradexist}) to
\begin{align*}
w^\prime(\lambda) = - \grad^2_1\fin(w(\lambda), \lambda)^{-1} \grad_{21}^2\fin(w(\lambda), \lambda). 
\end{align*}
Before starting with the study of the two methods \ITD\ and \IMD, 
we give a lemma which introduces
three additional constants that will occur in the complexity bounds. 
\begin{lemma}
\label{lm:boundEPhi}
Let $\lambda \in \Lambda$ and let $D_\lambda >0$ 
be such that $\norm{w(\lambda)} \leq D_\lambda$.
Then there exist $\Bo, \LPhi \in \R_{+}$ s.t.
\[
\sup_{\norm{w} \leq 2 \D} \hspace{-.22truecm}\lVert 
\grad_1 \fo({w},{\lambda}) \rVert \leq \Bo,\hspace{-.18truecm}~~\sup_{\norm{w} \leq 2\D}\hspace{-.18truecm}
\norm{ \jac_2 \Phi(w, \lambda) } \leq \LPhi\\
\]
\end{lemma}
The proof exploits the fact that the image of a continuous function applied to a compact set remains compact.


\subsection{Iterative Differentiation}\label{se:dynaDiff}

In this section
we replace $w(\lambda)$ in \eqref{mainprob} by the $t$-th 
 iterate of $\Phi(\cdot,\lambda)$, for which we additionally require the 
 following.
\begin{assumption}
\label{ass:contraction}
For every $\lambda \in \Lambda$,
$\Phi(\cdot, \lambda)$ is a contraction with constant $\q \in (0,1)$.
\end{assumption}

The approximation of the hypergradient $\nabla f(\lambda)$ is then obtained as in 
Algorithm~\ref{algo1}.
\begin{algorithm}[t]
\caption{Iterative Differentiation (ITD)}
\label{algo1}
\begin{enumerate}
\item
\label{ass:innerdynamic}
Let $t \in \N$, set $w_0(\lambda) = 0$, and compute, 
\vspace{-1ex}
\begin{equation*}
\begin{array}{l}
    \text{for}\;i=1,2,\ldots t\\[0.4ex]
    \left\lfloor
    \begin{array}{l}
    w_i(\lambda) = \Phi(w_{i-1}(\lambda), \lambda).
    \end{array}    
    \right.
\end{array}
\vspace{-1ex}
\end{equation*}
\item Set $f_t(\lambda) = E(\wt, \lambda)$.
\item Compute $\grad f_t(\lambda)$ using automatic differentiation.
    \vspace{-.25truecm}
\end{enumerate}
\end{algorithm}
Assumption~\ref{ass:contraction} looks quite restrictive, however it is satisfied in a number of interesting cases:
\vspace{-2ex}
\setlist[enumerate]{itemsep=0mm}
\begin{enumerate}[{\rm(a)}]
    \item 
\label{rmk:20200106b}
In the setting of the bilevel optimization problem~\eqref{mainprob1}, suppose that
for every $\lambda \in \Lambda$, $\ell(\cdot, \lambda)$ is $\mufi$-strongly convex and $\Llw$-Lipschitz smooth for some continuously differentiable functions $\mu_\ell\colon\Lambda \to \R_{++}$, and $L_\ell\colon \Lambda \to \R_{++}$. Set
 $\kappa(\lambda) = \Llw/\mufi$,
\begin{equation}
    \alpha(\lambda) = \frac{2}{\mufi + \Llw},
    \ \text{and}\ \ 
    \q=\frac{\kappa(\lambda) - 1}{\kappa(\lambda)+1}.
\end{equation}
Then, $\Phi(w, \lambda) = w - \alpha(\lambda) \nabla_1 \ell(w, \lambda)$ is a contraction w.r.t.~$w$ with 
constant $\q$ (see Appendix~\ref{appB}).
\item\label{caseb} For strongly convex quadratic functions, accelerated methods like Nesterov's \cite{nesterov1983method} or heavy-ball \cite{Polyak1987} can be formulated as fixed-point iterations of a contraction in  the norm defined by a suitable positive definite matrix. 
\item In certain graph and recurrent neural networks of the form \eqref{mainprob2}, where the transition function is 
assumed
to be a contraction \citep{scarselli2008graph,almeida1987learning,pineda1987generalization}.
\end{enumerate}

The following lemma is a simple consequence of the theory on Neumann series and shows that Assumption~\ref{ass:contraction} is stronger than Assumption~\ref{assA}\ref{ass:fixedpoint}-\ref{ass:invert}. For reader's convenience the proof is given in Appendix~\ref{sec:proofs}.

\begin{restatable}{lemma}{optnormcontractive}
\label{lm:optnormcontractive}
Let Assumption \ref{ass:contraction} be satisfied. Then, 
for every 
 $\lambda \in \Lambda$, 
 $\Phi(\cdot, \lambda)$ has a unique fixed point and, for every $w \in \R^d$,
 $I - \jac_1\Phi(w, \lambda)$ is invertible and
\begin{align*}
\norm{(I - \jac_1\Phi(w, \lambda))^{-1}} \leq \frac{1}{1 - \q}.
\end{align*}
In particular, \ref{ass:fixedpoint} and \ref{ass:invert} in Assumption~\ref{assA} hold. 
\end{restatable}

With Assumption \ref{ass:contraction}
in force and if $\wt$ is defined as at point
\ref{ass:innerdynamic} in Algorithm~\ref{algo1},
we have the following proposition that is essential for the final bound.

\begin{restatable}{proposition}{boundDiff}
\label{lm:boundDiff} 
Suppose that Assumptions \ref{assA}\ref{ass:philip} 
and \ref{ass:contraction}  hold
and let $t \in \N$, with $t\geq 1$.
Moreover, for every $\lambda \in \Lambda$,
let $\wt$
be computed by Algorithm~\ref{algo1}~
and let $\D$ and $\LPhi$ be as in Lemma~\ref{lm:boundEPhi}.
Then, $w_t(\cdot)$ is differentiable and,
for every $\lambda \in \Lambda$,
\begin{multline}
\label{eq:20200124e}
    \norm{ w^\prime_t(\lambda) - w^\prime(\lambda) } \\\leq \left(\rhol + \rhow\frac{\LPhi}{1 - \q}\right) \D t \q^{t-1} 
    +\frac{\LPhi}{1 - \q} \q^{t}.
\end{multline}
\end{restatable}

Leveraging Proposition~\ref{lm:boundDiff}, we give the main result of this section.

\begin{restatable}{theorem}{boundSolver} (ITD bound)
\label{boundSolver} 
Suppose that Assumptions
\ref{assA}\ref{ass:philip}-\ref{ass:foLip} and \ref{ass:contraction}
 hold and let $t \in \N$ with $t \geq 1$.
 Moreover, for every $\lambda \in \Lambda$, let
$\wt$ and $f_t$ be defined according to
Algorithm~\ref{algo1} and let $\D, \Bo$, 
and $\LPhi$ be as in Lemma~\ref{lm:boundEPhi}.
Then, $f_t$ is differentiable and, for every $\lambda \in \Lambda$,
\begin{multline}\label{eq:boundsolver}
\norm{\grad f_t (\lambda) - \grad f (\lambda)}
\leq \Big( c_1(\lambda)  + c_2(\lambda) \frac{t}{\q} + c_3 (\lambda) \Big) \q^t,
\end{multline}
where
\begin{align*}
    c_1(\lambda) &= \left(\Lol  + \frac{\Low \LPhi}{1-\q}\right) \D, \\
    c_2(\lambda) &= \left(\rhol + \frac{\rhow \LPhi}{1 - \q}\right)  \Bo\,\D, \\
    c_3 (\lambda) &= \frac{\Bo \,\LPhi}{1-\q}.
\end{align*}
\end{restatable}
In this generality this is a new result that provides a non-asymptotic linear rate of convergence for 
the gradient of $f_t$ towards that of $f$.


\subsection{Approximate Implicit Differentiation}\label{se:AID}

In this section we study another approach to approximate the gradient of $f$.
We derive from \eqref{eq:gradexist_1} and \eqref{eq:gradexist_2} that
\begin{equation}
\label{eq:20200123d}
    \grad f(\lambda) = \grad_2 \fo (w(\lambda), \lambda)  + \jac_2 \Phi(w(\lambda), \lambda)^\top  \Z
\end{equation}
where $\Z$ is the solution of the linear system
\begin{equation}
\label{eq:20200123a}
    (I - \jac_1\Phi(w(\lambda), \lambda)^\top)v = \grad_1 \fo (w(\lambda), \lambda).
\end{equation}
However,
in the above formulas
$w(\lambda)$ is usually 
not known explicitly or is
expensive to compute exactly.
To solve this issue $\nabla f(\lambda)$
is estimated as in Algorithm \ref{algo2}. Note that, unlike ITD, this procedure is agnostic about the algorithms used to compute the sequences $\wt$ and $\vk$. 
Interestingly, in the context of problem~\eqref{mainprob1}, choosing $\Phi(w,\lambda) = w - \grad_1 \fin(w, \lambda)$ in~\Cref{algo2} yields the same procedure studied by \citet{pedregosa2016hyperparameter}.
\begin{algorithm}[t]
\caption{Approximate Implicit Differentiation (AID)}
\label{algo2}
\begin{enumerate}
\item  Let $t \in \N$ and compute $\wt$ by $t$ steps of an algorithm converging to $w(\lambda)$, starting from $w_0(\lambda) = 0$.
\item Compute $\vk$ after $k$ steps of a solver for
the  system
\begin{equation}
\label{eq:linsystem}
    (I -\jac_1 \Phi(\wt, \lambda)^\top ) v = \grad_1 \fo (\wt, \lambda).
        \vspace{-.25truecm}
\end{equation}
\item Compute the approximate gradient as
\begin{equation*}
  \hat\grad{f}(\lambda)\!:=\! \grad_2 \fo (\wt, \lambda)  +  \jac_2 \Phi(\wt, \lambda)^\top\! \vk.
\vspace{-.45truecm}
\end{equation*}
\end{enumerate}

\end{algorithm}

The number of iterations $t$ and $k$ in Algorithm \ref{algo2} give a direct way of trading off accuracy and speed. To quantify this trade off we consider the following assumptions.

\begin{assumption}
\label{assC}
For every $\lambda \in \Lambda$,
\vspace{-1ex}
\begin{enumerate}[{\rm (i)}]

\item\label{ass:invertall} 
$\forall\, w \in \R^d$, $I -\jac_1 \Phi(w, \lambda)$ is invertible.

\item\label{ass:innerconv}  
$\norm{\wt - w(\lambda)} \leq \rf(t)\norm{w(\lambda)}$,  $\rf(t) \leq 1$, and $\rf(t) \to 0$ as $t \to +\infty$.

\item\label{ass:linapprox} 
$\norm{ \vk - \hZ } \leq \hrf(k) \norm{\hZ}$ and $\hrf(k) \to 0$ as $k \to +\infty$.
\end{enumerate}
\end{assumption}

If Assumption~\ref{assC}\ref{ass:invertall} holds,
then, for every $\lambda \in \Lambda$,
since the map $w \mapsto\norm{(I -\jac_1 \Phi(w, \lambda))^{-1}}$ is continuous, we have
\begin{equation}
\label{eq20200123c}
    \sup_{\norm{w} \leq 2\D} \norm{(I -\jac_1 \Phi(w, \lambda))^{-1}} \leq \frac{1}{\mu_\lambda}<+\infty,
\end{equation}
for some $\mu_\lambda>0$. We note that,
in view of Lemma~\ref{lm:optnormcontractive}, 
Assumption~\ref{ass:contraction} implies  
Assumption~\ref{assC}\ref{ass:invertall} (which in turn implies 
Assumption~\ref{assA}\ref{ass:invert}) and in \eqref{eq20200123c} one can take $\mu_\lambda = 1 - \q$. We stress that, Assumption~\ref{assC}\ref{ass:innerconv}-\ref{ass:linapprox} are general and do not specify the type of algorithms solving the fixed-point equation $w=\Phi(w,\lambda)$
and the liner system \eqref{eq:linsystem}. It is only required that such algorithms have explicit rates of convergence $\rf(t)$ and $\hrf(k)$ respectively. Finally, we note that
Assumption~\ref{assC}\ref{ass:innerconv} is less restrictive than Assumption~\ref{ass:contraction}
and encompasses the procedure at point  \ref{ass:innerdynamic} in Algorithm~\ref{algo1}: indeed in such case \ref{assC}\ref{ass:innerconv} holds with $\rf(t) = \q^t$.

It is also worth noting that the \IMD\ procedure requires only to store the last lower-level iterate, 
i.e. $\wt$. 
This is a considerable advantage over ITD, which instead requires to store the entire lower-level optimization trajectory $(w_i(\lambda))_{0 \leq i \leq t}$, if implemented using RMAD.

The iteration complexity bound for \IMD\ is given below.
This is a general bound which depends on the rate of convergence $\rf(t)$ of the sequence $(\wt)_{t \in \N}$ and the rate of convergence $\hrf(k)$ of the sequence $(\vk)_{k \in \N}$.

\begin{restatable}{theorem}{gradBoundGeneral} (AID bound)
\label{th:gradBoundImplicit} 
Suppose that Assumptions~\ref{assA}\ref{ass:fixedpoint}\ref{ass:philip}\ref{ass:foLip}  and \ref{assC}\ref{ass:invertall}--\ref{ass:linapprox} hold. 
Let $\lambda \in \Lambda$, $t \in \N$, $k \in \N$.
Let $\D, \Bo$, and $\LPhi$ be  as in Lemma~\ref{lm:boundEPhi}
and let $\mu_\lambda$ be defined according to \eqref{eq20200123c}.
Let $\hat{\nabla} f(\lambda)$ be defined as in Algorithm~\ref{algo2} and let $\hat\Delta =\norm{\hat\grad f(\lambda) - \grad f(\lambda)}$.
Then,
\begin{multline}\label{eq:boundInvertible}
\hat\Delta \leq \left( \Lol  + \frac{\Low\LPhi}{\pd} + \frac{\rhol \Bo}{\pd} + \frac{\rhow \LPhi \Bo}{\pd^2}  \right)\\
\times\D \rf(t) + \frac{\LPhi \Bo}{\pd} \hrf(k).
\end{multline}
Furthermore, if Assumption~\ref{ass:contraction} holds, then $\pd= 1-\q$ and
\begin{equation}\label{eq:boundInvertibleContraction}
\hat\Delta
\leq \Big( c_1(\lambda) + \frac{c_2(\lambda)}{1 -\q} \Big)\rf(t) + c_3(\lambda)\hrf(k).
\end{equation}
where $c_1(\lambda)$, $c_2(\lambda)$ and $c_3(\lambda)$ are defined in Theorem~\ref{boundSolver}.
\end{restatable}

\Cref{th:gradBoundImplicit} provides a non-asymptotic rate of convergence for 
$\hat{\nabla}f$ which contrasts with
the asymptotic result given in 
\citet{pedregosa2016hyperparameter}.
In this respect, making
Assumption~\ref{assC}\ref{ass:invertall}
instead of the weaker 
Assumption~\ref{assA}\ref{ass:invert} 
is critical.
 
Depending on the choice of the solver for the linear system \eqref{eq:linsystem}
different \IMD\ methods are obtained.
In the following we consider two cases.

\paragraph{\IMD\ with 
the Conjugate Gradient Method (AID-CG).}
For the sake of brevity
we set $\hA = I -\jac_1 \Phi(\wt, \lambda)^\top $
and $\hB =\grad_1 \fo (\wt, \lambda)$. Then,
the linear system \eqref{eq:linsystem} is equivalent to  the following minimization problem
\begin{equation}
\label{eq:minquadratic} 
    \min_{v\in \R^d} \frac{1}{2} \norm{\hA v - \hB}^2,  
\end{equation}
which, if $\jac_1 \Phi(\wt, \lambda)$ is symmetric (so that $\hA$ is also symmetric) is in turn equivalent to
\begin{equation}
\label{eq:minquadraticsym}
\min_{v\in \R^d} \frac{1}{2}v^\top \hA v - v^\top \hB. 
\end{equation}

Several first order methods solving problems  \eqref{eq:minquadratic} or \eqref{eq:minquadraticsym} satisfy assumption~\ref{assC}\ref{ass:linapprox} with linear rates and require only Jacobian-vector products. In particular, for the symmetric case \eqref{eq:minquadraticsym}, the \emph{conjugate gradient} method 
features the following linear rate
\begin{multline}
\label{eq:conjrate}
    \norm{\vk - v_t(\lambda)}\\
    \leq 2{\textstyle \sqrt{\kappa(\hA)}} \bigg(\!\frac{\sqrt{\kappa(\hA)} - 1}{\sqrt{\kappa(\hA)}+ 1}\bigg)^{\!k}\!\!
    \norm{\vo\! -\! v_t(\lambda)},
\end{multline}
where $\kappa(\hA)$ is the condition number of $\hA$.
In the setting
of case~\ref{rmk:20200106b} outlined in Section~\ref{se:dynaDiff}, 
$\hA = \alpha(\lambda)\nabla^2_1 \ell (\wt,\lambda)$
and 
$$\mufi I \preccurlyeq \nabla_1^2 \ell(\wt,\lambda) \preccurlyeq \Llw I.$$
Therefore the condition number of $\hA$ satisfies $\kappa(\hA) \leq \Llw/\mufi = \kappa(\lambda)$
and hence 
\begin{equation}
\label{eq:20200106d}
    \frac{\sqrt{\kappa(\hA)} - 1}{\sqrt{\kappa(\hA)}+ 1} \leq 
    \frac{\sqrt{\kappa(\lambda)}-1}
    {\sqrt{\kappa(\lambda)}+1} \leq \frac{\kappa(\lambda)-1}{\kappa(\lambda)+1} = \q.
\end{equation}


\paragraph{\IMD\ with the Fixed-Point Method (AID-FP).}
In this paragraph we
make a specific choice
for the sequence $(\vk)_{k \in \N}$ in Assumption~\ref{assC}\ref{ass:linapprox}. 
We let Assumption~\ref{ass:contraction} be satisfied and 
consider the following algorithm.
For every $\lambda\in \Lambda$ and 
$t\in \N$, we choose $\vo = 0 \in \R^d$ and, 
\vspace{-1ex}
\begin{equation}
\label{ass:fixedpointmethod} 
 \begin{array}{l}
\text{for}\;k=1,2,\ldots\\[0.4ex]
\left\lfloor
\begin{array}{l}
\vk =    \jac_1\Phi(\wt, \lambda)^\top \vkm \\[1ex]
\qquad\qquad\qquad\qquad\quad 
+ \grad_1 \fo(\wt, \lambda).
\end{array}    
\right.
\end{array}
\end{equation}

In such case the rate of convergence 
$\hrf(k)$ is linear. More precisely,
since $\norm{\jac_1\Phi(\wt, \lambda)} \leq \q<1$ (from Assumption~\ref{ass:contraction}), then
the mapping $$T\colon v \mapsto \jac_1\Phi(\wt, \lambda)v + \grad_1 \fo(\wt, \lambda)$$ is a contraction with constant $\q$. 
Moreover, the fixed-point of $T$ 
is the solution of 
\eqref{eq:linsystem}. Therefore,
$\norm{\vk - \hZ} \leq \q^k \norm{\vo - \hZ}$.
In the end the following result holds.
\begin{theorem}
\label{thm:fixedpoint} If Assumption~\ref{ass:contraction} holds and $(\vk)_{k \in \N}$ is defined according to \eqref{ass:fixedpointmethod},
then Assumption~\ref{assC}\ref{ass:linapprox} is satisfied with $\hrf(k) = \q^k$.
\end{theorem}

Now, plugging the rate $\hrf(k) = \q^k$
into the general bound 
\eqref{eq:boundInvertibleContraction} yields
\begin{equation}
\label{eq:20200128c}
\hat\Delta
\leq \Big( c_1(\lambda) +  \frac{c_2(\lambda)}{1 -\q} \Big)\rf(t) + c_3(\lambda)\q^k.
\end{equation}

However, an analysis similar to the one in
Section~\ref{se:dynaDiff} shows that the above result can be slightly improved as follows. 

\begin{restatable}{theorem}{boundFixed}
\label{boundFixed}
{(AID-FP bound)}
Suppose that 
Assumptions~\ref{assA}\ref{ass:fixedpoint}\ref{ass:philip}\ref{ass:foLip}
and Assumption~\ref{ass:contraction} hold. Suppose also that \eqref{ass:fixedpointmethod} holds. Let $\hat{\nabla} f(\lambda)$ be defined according to Algorithm~\ref{algo2} and $\hat\Delta =\norm{\hat\grad f(\lambda) - \grad f(\lambda)}$. Then,
for every $t,k \in \N$,
\begin{equation}\label{eq:boundfixed}
    \hat\Delta\leq  
    \Big( c_1(\lambda) + c_2(\lambda) \frac{1-\q^k}{1 -\q} \Big) \rf(t) + c_3(\lambda) \q^k,
\end{equation}
where $c_1(\lambda)$, $c_2(\lambda)$  and $c_3(\lambda)$ are given in Theorem~\ref{boundSolver}.
\end{restatable}

We end this section with a discussion about the consequences of the presented results.


\subsection{Discussion}\label{sec:discussion}

Theorem~\ref{th:gradBoundImplicit} shows that Algorithm~\ref{algo2} computes an approximate gradient of $f$ with a linear convergence rate (in $t$ and $k$), provided that the solvers for the lower-level problem and the linear system converge linearly. Furthermore, under Assumption~\ref{ass:contraction}, both AID-FP and \ITD\ converge linearly.
However, if in Algorithm~\ref{algo2}
 we define $\wt$ as at point \ref{ass:innerdynamic} in Algorithm~\ref{algo1} (so that $\rf(t) = \q^t$), 
and take $k=t$, 
then the bound for AID-FP \eqref{eq:boundfixed} is lower than that of \ITD{}  \eqref{eq:boundsolver}, since 
$\q(1 - \q^t)/(1-\q) = \sum_{i=1}^t \q^i < t$ for every $t\geq1$.
This analysis suggests that \IMD{}-FP  converge faster than \ITD{}.

We now discuss the choice of the algorithm to solve the linear system \eqref{eq:linsystem} in Algorithm~\ref{algo2}.
Theorem~\ref{boundFixed} 
provides a bound for AID-FP, which considers procedure \eqref{ass:fixedpointmethod}.
However, we see from~\eqref{eq:boundInvertibleContraction} in  Theorem~\ref{th:gradBoundImplicit}
that a solver for the linear system with rate of convergence $\hrf(k)$ faster than $\q^k$ may give a better bound.
The above discussion, together with \eqref{eq:conjrate} and \eqref{eq:20200106d}, proves that AID-CG
has a better asymptotic rate than AID-FP for instances of problem~\eqref{mainprob1} where the lower-level objective $\ell(\cdot,\lambda)$ is Lipschitz smooth and strongly convex (case~\ref{rmk:20200106b} outlined in Section~\ref{se:dynaDiff}).

\newcommand{\wstart}{w_{\text{start}}}
Finally, we note that both \ITD{} and \IMD{} consider the initialization  $w_0(\lambda) = 0$. However, in a gradient-based bilevel optimization algorithm, it might be more convenient to use a warm start strategy where $w_0(\lambda)$ is set based on previous upper-level iterations. Our analysis can be applied also in this case, but the related upper bounds will depend on the upper-level dynamics. This aspect makes it difficult to theoretically analyse the benefit of a warm start strategy, which remains an open question.

\section{Experiments}\label{sec:experiments}

\newcommand{\figdim}{.255}
\newcommand{\U}{\mathcal{U}}
\begin{figure*}[th]
    \centering
    \hspace{-.2truecm}
    \includegraphics[width=\figdim\textwidth]{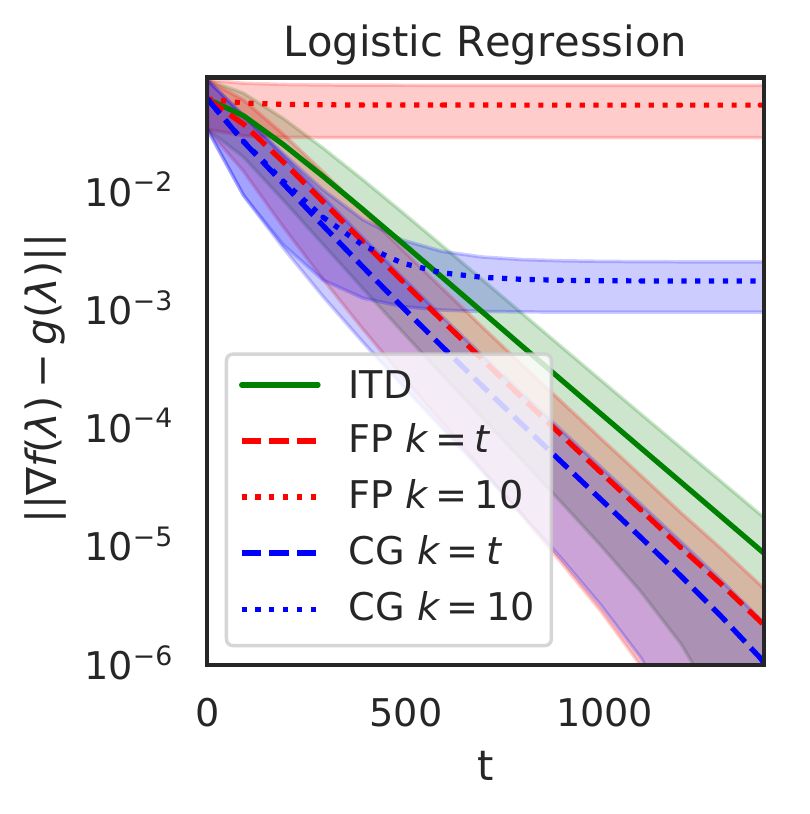}
    \hspace{-.25truecm}
    \includegraphics[width=\figdim\textwidth]{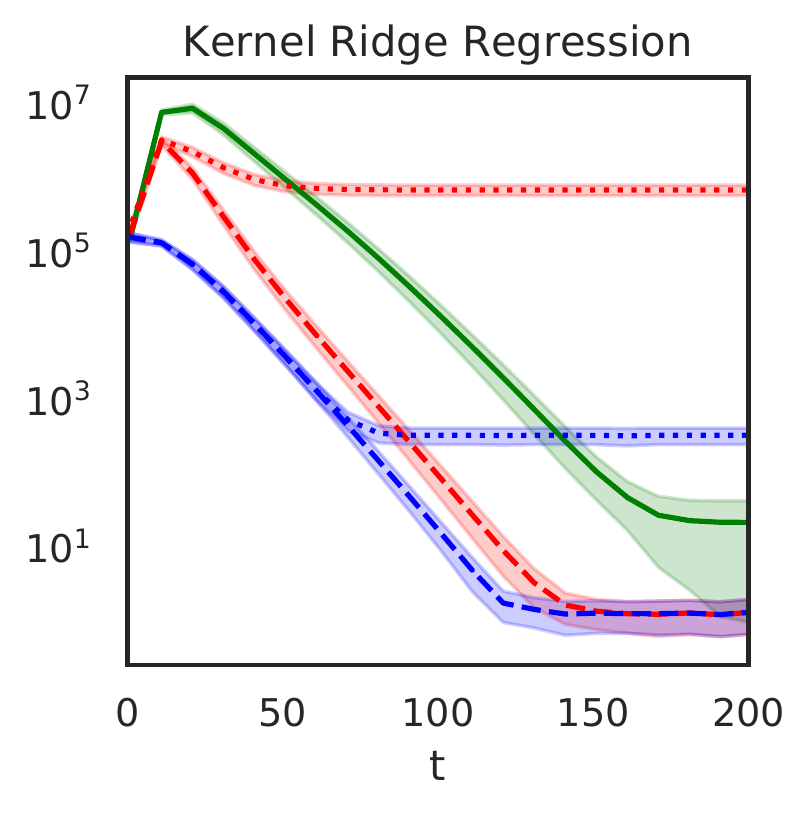}
    \hspace{-.25truecm}
    \includegraphics[width=\figdim\textwidth]{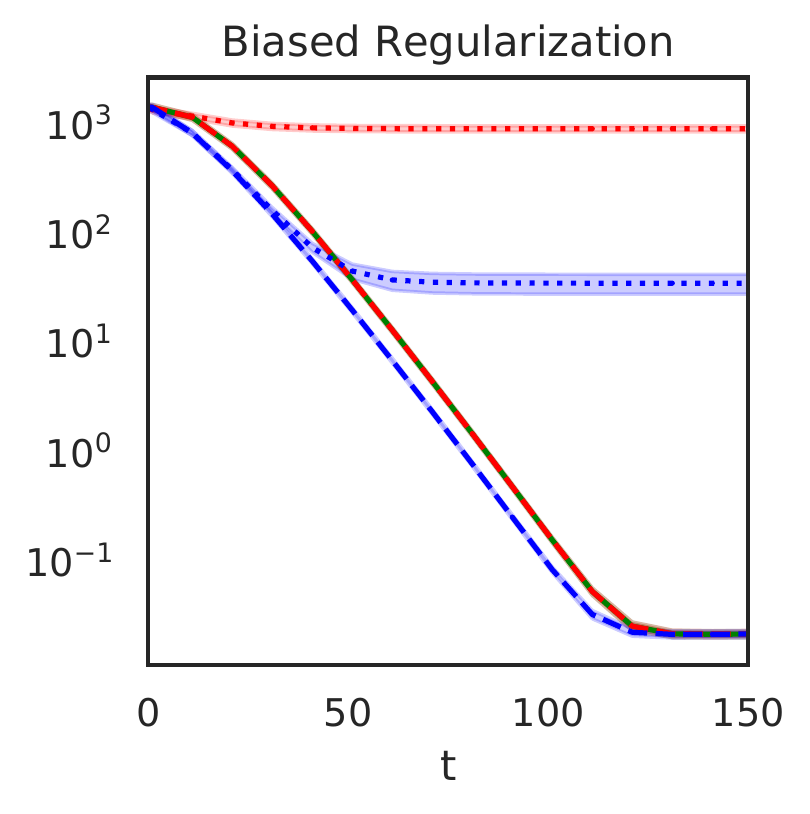}
    \hspace{-.25truecm}
    \includegraphics[width=\figdim\textwidth]{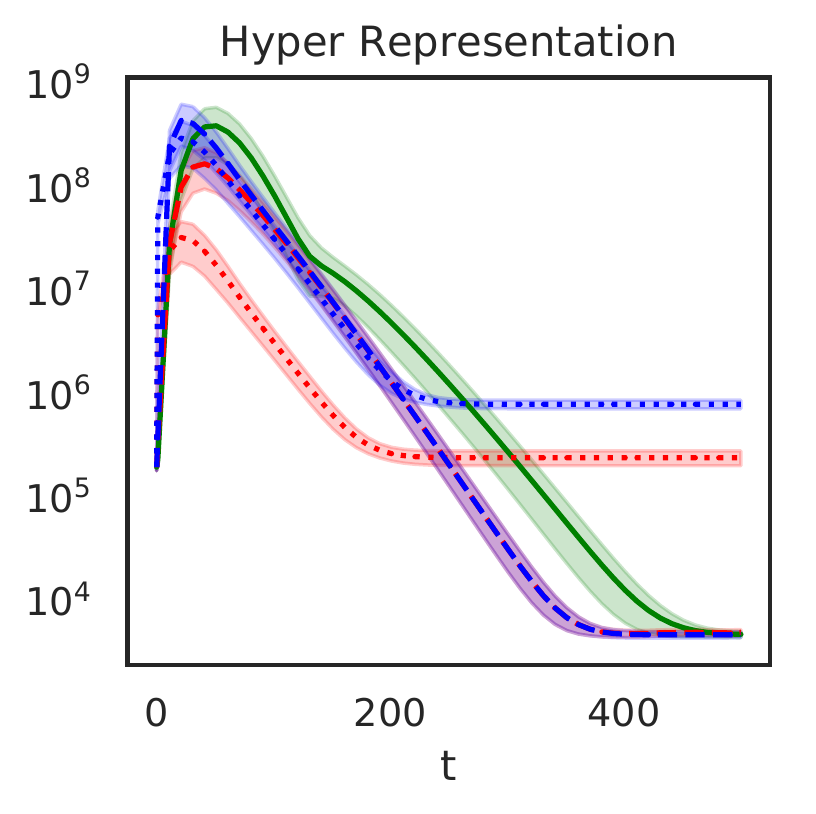}
    \vspace{-6mm}
    \caption{Convergence of different hypergradient approximations, where $g(\lambda)$ is equal to  $\grad f_t(\lambda)$ for \ITD{} and to $\hat\grad f(\lambda)$ for CG and FP.
    Mean and standard deviation (shaded areas) are computed over 20 values of $\lambda$ sampled uniformly from $[\lambda_{\min}, \lambda_{\max}]^n$. }
    \label{fig:synthlambdas}
    \vspace{-1mm}
\end{figure*}

In the first part of this section we
focus on the hypergradient approximation error
and show that the upper bounds presented in the previous section give a good estimate of the actual convergence behaviour of \ITD{} and \IMD{} strategies on a variety of settings. 
In the second part we present a series of experiments pertaining optimization on both the settings of hyperparameter optimization, as in problem~\eqref{mainprob1}, and learning equilibrium models, as in problem~\eqref{mainprob2}.
The algorithms have been implemented\footnote{
The code is freely available at the following link. 

\url{https://github.com/prolearner/hypertorch}} in PyTorch~\cite{NEURIPS2019_9015}.
In the following, we shorthand AID-FP and AID-CG with FP and CG, respectively.

\subsection{Hypergradient Approximation}
\label{sec:4.1}
In this section, we consider several problems of type~\eqref{mainprob1} with synthetic generated data (see \Cref{sec:datagen} for more details) where $\Dt = (\Xt, \yt)$ and $\Dv = (\Xv, \yv)$ are the training and validation sets respectively, with $\Xt \in \R^{n_e \times p}$, $\Xv \in \R^{n'_e \times p}$, being $n_e, n'_e$ the number of examples in each set and $p$ the number of features. 
Specifically we consider the following settings, which are representative instances of 
relevant bilevel problems in machine learning.

{\bfseries Logistic Regression with $\ell_2$ Regularization (LR).} This setting is similar to the one in \citet{pedregosa2016hyperparameter}, but we introduce multiple regularization parameters:
\begin{align*}
    f(\lambda) &= \sum_{(x_e,y_e) \in \Dv} \psi(y_e x_e^\top w(\lambda)), \\
    w(\lambda) &= \argmin_{w \in \R^p} \sum_{(x_e,y_e) \in \Dt} \psi(y_e x_e^\top w) + \frac{1}{2} w^\top {\rm diag}(\lambda) w,
\end{align*}
where $\lambda \in \R^p_{++}$, $\psi(x) = \log (1 + e^{-x})$ and ${\rm diag}(\lambda)$ is the diagonal matrix formed by the elements of $\lambda$.

{\bfseries Kernel Ridge Regression 
(KRR).} We extend the setting presented by \citet{pedregosa2016hyperparameter} considering a $p$-dimensional Gaussian kernel parameter $\gamma$ in place of the usual one: 
\vspace{-1ex}
\begin{equation}
\label{eq:krr}
\begin{aligned}
    f(\beta,\gamma) &= \frac{1}{2} \norm{ \yv - \kerv(\gamma) w(\beta,\gamma) }^2, \\
    w(\beta,\gamma) &= \argmin_{w \in \R^{n_e}} \frac{1}{2} w^\top\left(\kert(\gamma) + \beta I \right)w - w^\top \yt,
\end{aligned}
\vspace{-0.5ex}
\end{equation}
where $\beta \in (0,\infty)$, $\gamma \in \R^p_{++}$ and $\kerv(\gamma)$, $\kert(\gamma)$ 
are respectively the validation and training  kernel matrices (see \Cref{sec:datagen}).

\begin{table*}[th]
\small
\caption{Objective (test accuracy) values after $s$ gradient descent steps where $s$ is $1000$, $500$ and $4000$ for Parkinson, 20 newsgroup and Fashion MNIST respectively. Test accuracy values are in \%. $k_r=10$ for Parkinson and 20 newsgroup while for Fashion MNIST $k_r=5$.}\label{tb:objective}
\vspace{+1mm}
\begin{tabular}[t]{lrr}
\multicolumn{3}{c}{\textbf{Parkinson}} \\
\toprule
&   $t=100$ &   $t=150$ \\
\midrule
 ITD       & 2.39 (75.8) & 2.11 (69.7) \\
 FP $k=t$  & 2.37 (81.8) & 2.20 (77.3) \\
 CG $k=t$  & 2.37 (78.8) & 2.20 (77.3) \\
 FP $k=k_r$ & 2.71 (80.3) & 2.60 (78.8) \\
 CG $k=k_r$ & 2.33 (77.3) & 2.02 (77.3) \\
\bottomrule
\end{tabular}
\hfill
\begin{tabular}[t]{rrr}
\multicolumn{3}{c}{\textbf{20 newsgroup}} \\
\toprule
$t=10$ &   $t=25$ &   $t=50$ \\
\midrule
 1.08 (61.3) & 0.97 (62.8) & 0.89 (64.2) \\
 1.03 (62.1) & 1.02 (62.3) & 0.84 (64.4) \\
 0.93 (63.7) & 0.78 (63.3) & 0.64 (63.1) \\
 \multicolumn{1}{c}{$-$}  & 0.94 (63.6) & 0.97 (63.0) \\
  \multicolumn{1}{c}{$-$} & 0.82 (64.3) & 0.75 (64.2) \\
\bottomrule
\end{tabular}
\hfill
\begin{tabular}[t]{rrr}
\multicolumn{2}{c}{\textbf{Fashion MNIST}} \\
\toprule
$t=5$ &   $t=10$ \\
\midrule
0.41 (84.1) & 0.43 (83.8) \\ 
0.41 (84.1) & 0.43 (83.8) \\ 
0.42 (83.9) & 0.42 (84.0) \\ 
  \multicolumn{1}{c}{$-$} & 0.42 (83.9) \\ 
  \multicolumn{1}{c}{$-$} & 0.42 (84.0) \\ 
\bottomrule
\end{tabular}

\end{table*}


{\bfseries Biased Regularization (BR).}
Inspired by \citet{denevi2019learning,rajeswaran2019meta}, we consider the following.
\begin{equation}\label{eq:biasedreg}
\begin{aligned}
    f(\lambda) &= \frac{1}{2} \norm{\Xv w(\lambda) - \yv}^2, \\
    w(\lambda) &= \argmin_{w \in \R^p} \frac{1}{2} \norm{\Xt w - \yt}^2 + \frac{\beta}{2} \norm{w - \lambda}^2,
\end{aligned}
\nonumber
\end{equation}
where $\beta \in \R_{++}$ and $\lambda \in \R^p$. 

{\bfseries Hyper-representation (HR).}
The last setting, reminiscent of
 \citep{franceschi2018bilevel, bertinetto2019meta}, concerns learning a (common) linear transformation of the data and is formulated as 
\begin{equation}\label{eq:hr}
\begin{aligned}
    f(H) &= \frac{1}{2} \norm{\Xv H w(H) - \yv}^2 \\
    w(H) &= \argmin_{w \in \R^d} \frac{1}{2} \norm{\Xt H w - \yt}^2 + \frac{\beta}{2} \norm{w}^2
\end{aligned} 
\nonumber
\end{equation}
where $H \in \R^{p \times d}$ and  $\beta \in \R_{++}$.

LR and KRR are high dimensional extensions of classical hyperperparameter optimization problems, while BR and HR, 
are typically encountered in multi-task/meta-learning as single task objectives\footnote{In multi-task/meta-learning the upper-level objectives are averaged over multiple tasks and the hypergradient is simply the average of the single task one.}. 
Note that Assumption~\ref{ass:contraction} (i.e. $\Phi(\cdot,\lambda)$ is a contraction) can be satisfied for each of the aforesaid scenarios, since they all belong to case~\ref{rmk:20200106b} of Section~\ref{se:dynaDiff} (KRR, BR and HR also to case~\ref{caseb}).

We solve the lower-level problem in the same way for both \ITD{} and \IMD{} methods.
In particular, in LR we use the gradient descent method with optimal step size as in case~\ref{rmk:20200106b} of Section~\ref{se:dynaDiff}, while for the other cases we use the heavy-ball method with optimal step size and momentum constants. Note that this last method is not a contraction in the original norm, but only in a suitable norm depending on the lower-level problem itself.
To compute the exact hypergradient, we differentiate $f(\lambda)$ directly using RMAD  for KRR, BR and HR, where the closed form expression for $w(\lambda)$ is available, while for LR we use CG with $t=k = 2000$ in place of the (unavailable) analytic gradient. 

Figure~\ref{fig:synthlambdas} shows how the approximation error is affected by the number of lower-level iterations $t$. 
As suggested by the iteration complexity bounds in \Cref{sec:theory}, all the approximations, 
after a certain number of iterations, 
converge linearly to the true hypergradient\footnote{The asymptotic error can be quite large probably due to  numerical errors (more details in \Cref{app:exp}).}. Furthermore, in line with our analysis (see \Cref{sec:discussion}), CG gives the best gradient estimate (on average), followed by FP, while \ITD\ performs the worst. For HR, the error of all the methods increases significantly at the beginning, which can be explained by the fact that the heavy ball method
is not a contraction in the original norm and may diverge at first.
CG $k=10$ outperforms FP $k=10$ on 3 out of 4 settings but both remain far from convergence.

\subsection{Bilevel Optimization}
\label{sec:4.2}
In this section, we aim to solve instances of the bilevel problem~\eqref{mainprob} in which 
$\lambda$ has a high dimensionality.  

{\bfseries Kernel Ridge Regression on Parkinson.}
We take $f(\beta,\gamma)$ as defined in problem~\eqref{eq:krr} where the data is taken from the UCI Parkinson dataset~\cite{little2008suitability}, containing 195 biomedical voice measurements (22 features) from people with Parkinson's disease. To avoid projections, we replace $\beta$ and $\gamma$ respectively with $\exp(\beta)$ and $\exp(\gamma)$ in the RHS of the two equations in~\eqref{eq:krr}.  We split the data randomly into three equal parts to make the train, validation and test sets. 

{\bfseries Logistic Regression on 20 Newsgroup\footnote{http://qwone.com/~jason/20Newsgroups/}.}  
This dataset contains 18000 news divided in 20 topics and the features consist in 101631 tf-idf sparse vectors. We split the data randomly into three equal parts for training, validation and testing. We aim to solve the bilevel problem
\begin{equation}
\begin{gathered}
\min_{\lambda \in \R^{p}} \mathrm{CE}(\Xv w(\lambda), \yv) \\
w(\lambda) = \argmin_{w \in \R^{p \times c}} \mathrm{CE}(\Xt w, \yt) + \frac{1}{2cp}\sum_{i=1}^c \sum_{j=1}^p \exp(\lambda_j) w_{ij}^2 
\end{gathered}
\nonumber 
\end{equation}
where $\mathrm{CE}$ is the average cross-entropy loss,  
$c=20$ and $p=101631$. To improve the performance, we use warm-starts on the lower-level problem, i.e. we take $w_0(\lambda_i) = w_t(\lambda_{i-1})$ for all methods, 
where $(\lambda_i)_{i=1}^s$ are the upper-level iterates. 

{\bfseries Training Data Optimization on Fashion MNIST.}
Similarly to \cite{maclaurin2015gradient}, 
we optimize the features of a set of 10 training points,
each with a different class label on the Fashion MNIST dataset~\cite{xiao2017/online}.
More specifically we define the bilevel problem as
\begin{equation}
\begin{gathered}
\min_{\Xt \in \R^{c \times p}} \mathrm{CE}(\Xv w(\Xt), \yv) \\[-0.8ex]
w(\Xt) = \arg\min_{w \in \R^{p \times c}} \mathrm{CE}(\Xt w, \yt) + \frac{\beta}{2cp} \norm{ w }^2
\end{gathered}
\nonumber 
\end{equation}
where $\beta = 1$, $c=10$, $p=784$, $y = (0, \dots, c)^\top$ and $(\Xv, \yv)$ contains the usual training set.

We solve each problem using  (hyper)gradient descent with fixed step size selected via grid search (additional details are provided in \Cref{app:bopt}). The results in~\Cref{tb:objective} show
the upper-level objective and test accuracy both computed on the approximate lower-level  solution $w_t(\lambda)$ after bilevel optimization\footnote{For completeness, we also report in the Appendix (\Cref{tb:trueobj})  the upper-level objective and test accuracy both computed on the exact lower-level  solution $w(\lambda)$.}.
For Parkinson and Fashion MNIST, there is little difference among the methods for a fixed $t$. For 20 newsgroup, CG $k=t$ reaches the lowest objective value, followed by CG $k=10$. We recall that for \ITD\ we have cost in memory which is linear in $t$ and that, in the case of 20 newsgroup for some $t$ between $50$ and $100$, this cost exceeded the 11GB on the GPU. AID methods instead, require little memory and, by setting $k < t$, yield  similar or even better performance at a lower computation time.
Finally, we stress that since the upper-level objective is nonconvex, 
possibly with several minima,
gradient descent with a more precise estimate of the hypergradient may get more easily trapped in a bad local minima.

\begin{figure*}[t]
    \centering
    \includegraphics[width=\textwidth]{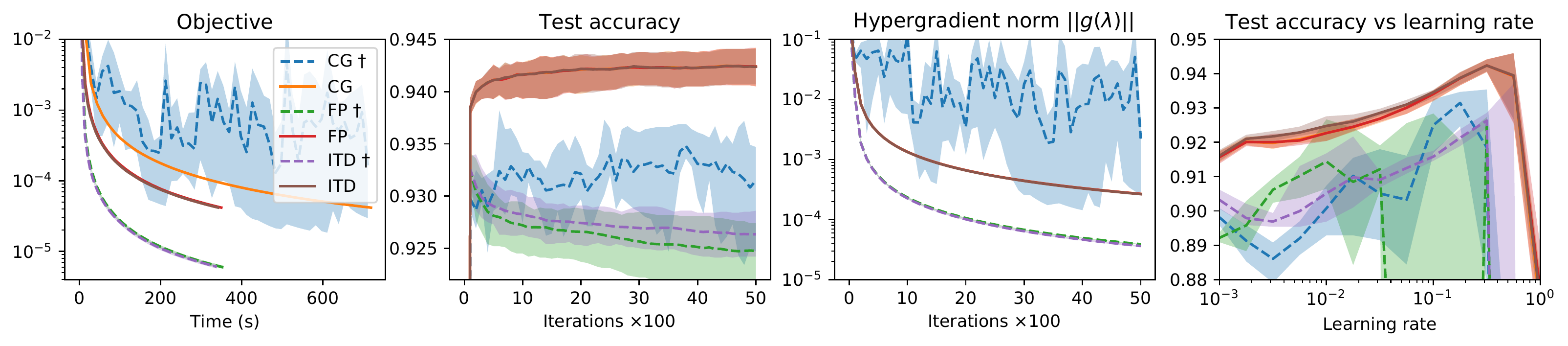}
    \vspace{-5mm}
    \caption{Experiments on EQM problems.  Mean (solid or dashed lines) and point-wise minimum-maximum range (shaded regions) across 5 random seeds that only control the initialization of $\lambda$.
    The estimated hypergradient $g(\lambda)$ is 
    equal to $\nabla f_t(\lambda)$ for \ITD{} and $\hat{\nabla} f(\lambda)$ for \IMD.
    We used $t=k=20$ for all methods and Nesterov momentum for optimizing $\lambda$, applying a projection operator at each iteration except for the methods marked with $\dag$. When performing projection, the curves produced by the three approximation schemes mostly overlap, indicating essentially the same performance (although at a different computational cost).
    }
    \label{fig:eqm}
\end{figure*}

{\bfseries Equilibrium Models.} 
Our last set of experiments investigates the behaviour of the hypergradient approximation methods on a simple instance of EQM (see problem~\eqref{mainprob2}) on non-structured data. EQM are an attractive class of models due to their mathematical simplicity, enhanced interpretability 
and memory efficiency. 
A number of works \citep{miller2019stable, bai2019deep} have recently shown that EQMs can perform on par with standard deep nets on a variety of complex tasks, renewing the interest in these kind of models. 

We use a subset of $n_e=5000$ instances randomly sampled from the MNIST dataset as training data and employ a multiclass logistic classifier paired with a cross-entropy loss.
We picked a small training set and purposefully avoided stochastic optimization methods to better focus on issues related to the computation of the hypergradients itself, avoiding the introduction of other sources of noise. 
We parametrize $\phi_i$ as 
\begin{equation}
\label{eq:qlin_eqm}
\phi_i(w_i, \gamma) = \tanh(A w_i + B x_i + c) \quad \text{for} \; 1\leq i \leq n_e
\end{equation}
where $x_i\in\R^{p}$ is the $i$-th example, $w_i\in\mathbb{R}^h$ and $\gamma=(A, B, C)\in
\mathbb{R}^{h\times h}\times \mathbb{R}^{h\times p} \times \mathbb{R}^h$. 
Such a model may be viewed as a (infinite layers) feed-forward neural network with tied weights or as a recurrent neural network with static inputs. 
Additional experiments with convolutional equilibrium models may be found in Appendix \ref{sec:apx:eqmconv}.
Imposing $\norm{A}<1$ ensures that the transition functions \eqref{eq:qlin_eqm}, and hence $\Phi$, are contractions. This can be  achieved during
optimization by projecting the singular values of $A$ onto the interval $[0,1-\varepsilon]$ for $\varepsilon>0$. 
We note that regularizing the norm of $\partial_1\phi_i$ or adding $L_1$ or $L_{\infty}$ penalty terms on $A$ may encourage, but does not strictly enforce, $\norm{A}<1$.

We conducted a series of experiments to ascertain the importance of the contractiveness of the map $\Phi$, as well as to understand which of the analysed methods is to be preferred in this setting. Since here $\partial_1 \Phi$ is not symmetric, the conjugate gradient method must be applied on the normal equations of problem~\eqref{eq:minquadratic}. We set $h=200$ and use $t=20$ fixed-point iterations to solve the lower-level problem in all the experiments. The first three plots of Figure \ref{fig:eqm} report training objectives, test accuracies and 
 norms of the estimated hypergradient for 
 each of the
 three methods, either applying or not the constraint on $A$, while the last explores the sensitivity of the methods to the choice of the learning rate. Unconstrained runs are marked with $\dag$. 
 Referring to the rightmost plot, it is clear (large shaded regions) that not constraining the spectral norm results in unstable behaviour of the ``memory-less'' \IMD{} methods (green and blue lines) for all but a few learning rates, 
while \ITD{} (violet), as expected, suffers comparatively less.
On the contrary, when $\norm{A}<1$ is enforced, all the approximation methods are successful and stable, with FP to be preferred being faster then CG on the normal equations and requiring substantially less memory than \ITD. 
As a side note, referring to Figure \ref{fig:eqm} left and center-left, we observe that projecting onto the spectral ball acts as powerful regularizer, in line with the findings of \citet{sedghi2019singular}. 



\section{Conclusions}
We studied a general class of bilevel problems where at the lower-level we seek for a solution to a parametric fixed point equation. This formulation encompasses several learning algorithms recently considered in the literature. We established results on the iteration complexity of two strategies  to compute the hypergradient (ITD and AID) under the assumption that the fixed point equation is defined by a contraction mapping.
Our practical experience with these methods on a number of bilevel problems indicates that there is a trade-off between the methods, with AID based on the conjugate gradient method being preferable due to a potentiality better approximation of the hypergradient and lower space complexity.
When the contraction assumption is not satisfied, however, our experiments on equilibrium models suggest that ITD is more reliable than AID methods.  
In the future, it would be valuable to extend the ideas presented here to other challenging machine learning scenarios not covered by our theoretical analysis. These include bilevel problems in which the lower-level is only locally contactive, nonsmooth, possibly nonexpansive or can only be solved via a stochastic procedure.
At the same time, there is a need to clarify the tightness of the iteration complexity bounds presented here.

\paragraph{Acknowledgments.} This work was supported in part by a grant from SAP SE.
\bibliography{convhg}
\bibliographystyle{icml2020}



\newpage
\appendix

\newpage
\clearpage



{\huge\textbf{Appendix}}\\

The Appendix is organized as follows.
\begin{itemize}
    \item Appendix~\ref{sec:proofs} presents the proofs of the results stated in \Cref{sec:theory}.   
    \item Appendix~\ref{app:gd} specializes the bounds in \Cref{sec:theory} in the case where the lower-level solution can be written as the fixed point of a one step gradient descent map.
    \item Appendix~\ref{app:exp} presents the details of the experiments in \Cref{sec:experiments} and additional results.
\end{itemize}

\section{Proofs of the Results in Section~\ref{sec:theory}}
\label{sec:proofs}
In this section we provide complete proofs of the results presented in the main body, which are restated here for the convenience of the reader. We also report few necessary additional results. 
\begin{theorem}
\label{th:gradexist} (Differentiability of $f$). 
Consider problem \eqref{mainprob} and suppose that Assumption~\ref{assA}\ref{ass:fixedpoint}-\ref{ass:invert} holds.
Then $ w(\cdot)$ and $f(\cdot)$ are differentiable 
on $\Lambda$ and, for every $\lambda \in \Lambda$
\begin{align}
\label{eq:Agradexist_1}
w'(\lambda) &= (I - \jac_1\Phi(w(\lambda), \lambda))^{-1}\jac_2 \Phi(w(\lambda), \lambda) \\
\label{eq:Agradexist_2}\grad{f}(\lambda) &= \grad_2 \fo (w(\lambda), \lambda)  + w'(\lambda)^\top \grad_1 \fo (w(\lambda), \lambda).
\end{align}
\end{theorem}
\begin{proof}
The function $G(w, \lambda) := w - \Phi(w, \lambda)$ is continuously differentiable on $\R^d \times \Lambda$. Then, we have 
\begin{equation*}
    \jac_1 G(w(\lambda), \lambda) = I - \jac_1\Phi(w(\lambda), \lambda),
\end{equation*}
which is  invertible due to Assumption~\ref{assA}\ref{ass:invert}.
Thus, since $G(w(\lambda), \lambda) = 0$, the implicit function theorem yields that $w(\lambda)$ is continuously differentiable with derivative
\begin{align*}
    w^\prime(\lambda) &=\jac_1 G(w(\lambda), \lambda)^{-1} \jac_2 G(w(\lambda), \lambda)\\
    &= (I - \jac_1\Phi(w(\lambda), \lambda))^{-1} \jac_2\Phi(w(\lambda), \lambda).
\end{align*}
Finally, 
\eqref{eq:Agradexist_2} follows from the chain rule for differentiation.
\end{proof}

\begin{corollary}
\label{cor:gradexist}
Suppose that in problem \eqref{mainprob1}, the
function $\fin\colon\R^d \times \Lambda \to \R$ is twice continuously differentiable 
and strongly convex w.r.t.~the first variable.
Let $\alpha\colon \Lambda \to \R_{++}$
be a differentiable function. Then the conclusions of Theorem~\ref{th:gradexist} hold and
\begin{align*}
(\forall\, \lambda \in \Lambda)\quad w^\prime(\lambda) = - \grad^2_1\fin(w, \lambda)^{-1} \grad_{21}^2\fin(w(\lambda), \lambda). 
\end{align*}
\end{corollary}
\begin{proof}
Define 
$\Phi (w,\lambda) =  w - \alpha(\lambda)\grad_1\fin(w, \lambda)$. Then,
Fermat's rule for the lower-problem in \eqref{mainprob1} yields that
$w(\lambda)$ is a fixed point for $\Phi(\cdot,\lambda)$, while $I - \jac_1\Phi(w(\lambda), \lambda)$ is invertible since
\begin{equation*}
I - \jac_1\Phi(w(\lambda), \lambda) 
= \alpha(\lambda)\grad^2_1\fin(w, \lambda)
\end{equation*}
and $\alpha(\lambda) \neq 0$. Therefore, 
Theorem~\ref{th:gradexist} applies and, since $\jac_2 \Phi (w(\lambda), \lambda) = - \alpha(\lambda) \grad^2_{21} \fin(w(\lambda), \lambda)$,
\eqref{eq:Agradexist_1} yields $w^\prime(\lambda) = - \frac{\alpha(\lambda)}{\alpha(\lambda)}\grad^2_1\fin(w, \lambda)^{-1} \grad_{21}^2\fin(w(\lambda), \lambda)$.
\end{proof}

\optnormcontractive*
\begin{proof}
Let $\lambda \in \Lambda$ and $w \in \R^d$.
Since $\Phi(\cdot, \lambda)$ is Lipschitz continuous with constant $q_\lambda$, it follows that 
\begin{equation}
\label{eq:20191223a}
\norm{\partial_1 \Phi(w,\lambda)} \leq q_\lambda<1.
\end{equation}
Therefore,
\begin{equation*}
\sum_{k=0}^\infty \norm{\jac_1\Phi(w, \lambda)}^k \leq \sum_{k=0}^\infty \q^k 
= \frac{1}{1-\q}.
\end{equation*}
Thus, $I - \jac_1\Phi(w, \lambda)$ is invertible and $\sum_{k=0}^\infty \jac_1\Phi(w, \lambda)^k = (I - \jac_1\Phi(w, \lambda))^{-1}$
and the bound follows.
\end{proof}

In the following technical lemma  we give two results which are fundamental for the proofs of the ITD bound (\Cref{boundSolver}) and the AID-FP bound (\Cref{boundFixed}).
The first result is standard
(see \cite{Polyak1987}, Lemma~1, Section~2.2).

\begin{lemma}
\label{lem:sequence}
Let $(u_k)_{k \in \N}$ and $(\tau_k)_{k \in \N}$ 
be two sequences of real non-negative numbers and let $q \in [0,\infty)$. Suppose that, for every $k \in \N$, with $k\geq 1$,
\begin{equation}
    u_{k} \leq q u_{k-1} + \tau_{k-1}.
\end{equation}
Then, the following hold.
\begin{enumerate}[(i)]
    \item\label{eq:20200124f} If $(\tau_k)_{k \in \N} \equiv \tau$, then $u_k \leq q^k u_0 + \tau(1-q^k)/(1-q)$.
    \item\label{eq:20200124g} If, for every integer $k\geq 1$, $\tau_k \leq q \tau_{k-1}$, then
    $u_k \leq q^k u_0 + k q^{k-1}\tau_0$.
\end{enumerate}
\end{lemma}
\begin{proof}
Let $k \in \N$, with $k \geq 1$.
Then, we have
\begin{align}
\nonumber u_{k} &\leq q u_{k-1} + \tau_{k-1}\\
\nonumber& \leq q (q u_{k-2} + \tau_{k-2}) + \tau_{k-1}\\
\nonumber& = q^2 u_{k-2} + (\tau_{k-1} + q \tau_{k-2})\\
&\;\;\vdots \notag \\
\label{eq:20200127a}&\leq q^k u_0 
+ \sum_{i=0}^{k-1} q^i \tau_{k-1-i}.
\end{align}

\ref{eq:20200124f}: Suppose that $(\tau_k)_{k \in \N} \equiv \tau$. Then it follows from \eqref{eq:20200127a} that $u_k \leq q^k u_0 + \tau \sum_{i=0}^{k-1} q^i =  q^k u_0 + \tau (1-q^k)/(1-q)$.

\ref{eq:20200124g}:
Suppose that, for every integer $k \geq 1$,
$\tau_{k} \leq q \tau_{k-1}$. Then, for every 
integers $k,i$ with $i \leq k-1$, we have
$\tau_{k-1- i} \leq q^{k-1-i} \tau_0$, which substituted into \eqref{eq:20200127a} yields
\begin{equation*}
    u_k \leq q^k u_0 + \sum_{i=0}^{k-1} q^i q^{k-1-i} \tau_0
\end{equation*}
and \ref{eq:20200124g} follows.
\end{proof}

\boundDiff*
\begin{proof}
We assume that $(\wt)_{t \in \N}$ is defined through the iteration
\begin{equation}
\label{ass:innerdynamic2}
\wt = \Phi(w_{t-1}(\lambda),\lambda)
\end{equation}
starting from $w_0(\lambda)=0 \in \R^d$.
Let $t \in \N$ with $t\geq 1$. Then,
the mapping
$\lambda \mapsto w_t(\lambda)$ is differentiable since, in view of \eqref{ass:innerdynamic}, it is a composition of differentiable functions, whereas $w^\prime(\lambda)$ exists due to Theorem~\ref{th:gradexist}. 
Differentiating the lower-level equation in \eqref{mainprob} and the recursive equation in \eqref{ass:innerdynamic2},
we get 
\begin{align}
\nonumber w^\prime_{t}(\lambda) &=   \jac_1\Phi(w_{t-1}(\lambda), \lambda) w'_{t-1}(\lambda)  + \jac_2\Phi(w_{t-1}(\lambda), \lambda)\\
\label{eq:20200224b}
w^\prime(\lambda) &=   \jac_1\Phi(w(\lambda), \lambda) w^\prime(\lambda)  + \jac_2\Phi(w(\lambda), \lambda).
\end{align}
Therefore, we get
\begin{align*}
    \lVert w^\prime_{t}(\lambda) &-
    w^\prime(\lambda)\rVert  \\ 
    &\leq 
     \norm{\jac_1\Phi ( w_{t-1}(\lambda), \lambda)- \jac_1\Phi ( w(\lambda), \lambda)} \norm{w^\prime(\lambda)} \\
    &\qquad + \norm{\jac_1\Phi ( w_{t-1}(\lambda), \lambda)}
    \norm{w^\prime_{t-1}(\lambda)-w^\prime(\lambda) } \\
    & \qquad +\norm{\jac_2\Phi ( w_{t-1}(\lambda), \lambda)- \jac_2\Phi ( w(\lambda), \lambda)}
\end{align*}
and hence, we derive from Assumption \ref{assA}\ref{ass:philip}, Assumption~\ref{ass:contraction}, equation \eqref{eq:gradexist_1} and Lemmas \ref{lm:boundEPhi} and \ref{lm:optnormcontractive}, that
\begin{align*}
    \lVert w^\prime_{t}(\lambda) &- w^\prime(\lambda) \rVert \\
    &\leq  (\rhol + \rhow\LPhi/(1-\q)) \norm{w_{t-1}(\lambda) - w(\lambda) }\\
    &\qquad + \q \norm{  w^\prime_{t-1}(\lambda) - w^\prime(\lambda)}.
\end{align*}
 Then, setting $p :=  \rhol + \rhow\LPhi/(1-\q)$, $\Delta_{t} : = \norm{ w_{t}(\lambda) - w(\lambda)}$ 
 and $\Delta^\prime_t := \norm{  w^\prime_{t}(\lambda) - w^\prime(\lambda)}$,
 we get
 \begin{equation*}
  \Delta_{t} \leq \q \Delta_{t-1}  
  \quad\text{and}\quad
  \Delta^\prime_t \leq \q \Delta^\prime_{t-1} + p \Delta_{t-1}.
 \end{equation*}
 Therefore, it follows from 
 Lemma~\ref{lem:sequence}\ref{eq:20200124g}
 (with $u_t = \Delta^\prime_t$ and $\tau_t = p \Delta_{t}$) that
\begin{equation*}
\Delta^\prime_t 
\leq \q^t\Delta^\prime_0 +  t \q^{t-1} p \Delta_0
\leq\frac{\LPhi}{1 - \q} \q^{t} + p \D t \q^{t-1},
\end{equation*}
where in the last inequality we used the bounds (see
\eqref{eq:20200224b} and Lemmas~\ref{lm:boundEPhi} and \ref{lm:optnormcontractive})
\begin{align}
\nonumber \Delta_0 &= \norm{w(\lambda) - w_0(\lambda)} = \norm{w(\lambda)} \leq \D \\
\label{eq:20200124c}\Delta^\prime_0 &= \norm{w^\prime(\lambda) - w^\prime_0(\lambda)} = \norm{w^\prime(\lambda)}  
    \leq \frac{\LPhi}{1 - \q}. 
\end{align}
Recalling the definitions of $p$ and $\Delta^\prime_t$, \eqref{eq:20200124e} follows.
\end{proof}

\boundSolver*
\begin{proof}
It follows from the definitions of $f_t$
and $f$ in Algorithm~\ref{algo1} and \eqref{mainprob} respectively and the chain rule for differentiation that
\begin{align*}
 \grad f_t(\lambda) &= \grad_2 \fo (w_t(\lambda), \lambda)  + w'_t(\lambda)  ^\top\grad_1 \fo (w_t(\lambda), \lambda)\\
 \grad f(\lambda) &= \grad_2 \fo (w(\lambda), \lambda) +  w^\prime(\lambda)^\top \grad_1\fo (w(\lambda), \lambda).
\end{align*}
Therefore, 
\begin{align*}
    \lVert \grad f_t&(\lambda) - \grad f(\lambda)\rVert \\
    &\leq  \norm{\grad_2 \fo(w_t(\lambda), \lambda) - \grad_2 \fo(w(\lambda), \lambda)} \\
    &\qquad +\norm{w^\prime(\lambda)} \norm{\grad_1 \fo(w_t(\lambda), \lambda) - \grad_1 \fo(w(\lambda), \lambda)} \\
    &\qquad + \norm{w^\prime_t(\lambda) - w^\prime(\lambda)} \norm{\grad_1 \fo(w_t(\lambda), \lambda)}.
\end{align*}
Now, we note that
$\norm{w_t(\lambda)} \leq \norm{w_t(\lambda) - w(\lambda)} + \norm{w(\lambda)} \leq (\q^t + 1) \norm{w(\lambda)} \leq 2 \D$.
Therefore, it follows from Assumption~\ref{assA}\ref{ass:foLip} and Lemmas~\ref{lm:boundEPhi} and \ref{lm:optnormcontractive} that
\begin{align*}
    \norm{\grad f_t (\lambda) - \grad f (\lambda)} \leq& \left(\Lol + \Low \LPhi / (1 - \q)\right) \q^t \D \\ 
    &+ \Bo \norm{w^\prime(\lambda) - w'_t(\lambda)},
\end{align*}
where we used $\norm{w_t(\lambda) - w(\lambda)} \leq \q^t \norm{ w_0(\lambda) - w(\lambda)} 
= \q^t\norm{w(\lambda)} \leq \q^t \D$. Then, \eqref{eq:boundsolver} follows 
from Proposition~\ref{lm:boundDiff}.
\end{proof}

Now we address the proofs related to the \IMD\ method described in Section~\ref{se:AID}.

\gradBoundGeneral*
\begin{proof}
For the sake of brevity we set
\begin{align*}
    \hA &= I - \jac_1 \Phi (\wt, \lambda)^\top\!\!\!,
    &\A &= I - \jac_1\Phi (w(\lambda),\lambda)^\top\!\!\!,
\\
    \hAl &= \jac_2 \Phi (\wt, \lambda),
    & \Al &= \jac_2 \Phi (w(\lambda), \lambda), 
\\
    \hB &= \grad_1 \fo(\wt, \lambda), 
    & \B &= \grad_1 \fo(w(\lambda), \lambda), 
\\
    \hC &= \grad_2 \fo(\wt, \lambda),
    &\C &= \grad_2 \fo(w(\lambda), \lambda).
\end{align*}
It follows form Assumption~\ref{assC}\ref{ass:innerconv} that
$\norm{\wt - w(\lambda)} \leq \rf(t) \norm{w(\lambda)} \leq \rf(t) \D \leq \D$ and hence
$\norm{\wt} \leq \norm{\wt - w(\lambda)} + \norm{w(\lambda)} \leq 2\D$.
Then the following upper bounds related to the 
 above quantities  follow from Assumptions~\ref{assA}\ref{ass:philip}-\ref{ass:foLip}, equation \eqref{eq20200123c} and Lemma~\ref{lm:boundEPhi}.
\begin{align*}
&\norm{\hA\! -\! \A} \leq \rhow \rf(t)\D,\ 
\norm{\hAl \!-\! \Al} \leq \rhol \rf(t)\D,\\
&\norm{\hB - \B} \leq \Low\rf(t)\D,\ 
\norm{\hC -\C} \leq \Lol \rf(t)\D,\\
&\norm{\hA^{-1}}, \norm{\A^{-1}} \leq \frac{1}{\pd},\ \norm{\hAl} \leq \LPhi,
\ \norm{\hB},\norm{\B} \leq \Bo.
\end{align*}
Now, setting $\hZ = \hA^{-1}\hB$ and $\Z = \A^{-1}\B$, $\hat \grad f(\lambda)$\footnote{see point 3 in \Cref{algo2}.} and
\eqref{eq:20200123d} can be written as
\begin{equation*}
\hat\grad f(\lambda) = \hC + \hAl^\top \vk,\ 
\grad f(\lambda) = \C + \Al^\top\Z.
\end{equation*}
Then, we have
\begin{align*}
    \lVert &\hat\grad f(\lambda) - \grad f(\lambda) \rVert \\
    &= \norm{\hC + \hAl^\top \vk - \C - \Al^\top\Z} \\
    &\leq \norm{\hC - \C}
    \\ &\quad\ +\norm{\hAl^\top\vk - \hAl^\top\Z + \hAl^\top\Z  - \Al^\top\Z} \\
    &\leq \norm{\hC - \C}\\ 
    &\quad\ + \norm{\hAl} \norm{\vk - \Z} 
    + \norm{\hAl - \Al} \norm{\Z} \\
    &\leq \norm{\hC - \C} \\
    &\quad\ + \norm{\hAl}
    \norm{\vk - \Z}\!+\! \norm{\hAl - \Al} \norm{\A^{-1}} \norm{\B} \\
    &\leq \left(\Lol + \frac{\rhol \Bo}{\pd}  \right) \rf(t) \D + \LPhi \norm{\vk - \Z}.
\end{align*}
Moreover, it follows from \ref{assC}\ref{ass:linapprox} that
\begin{align*}
    \norm{\vk - \Z} 
    &\leq  \norm{\vk - \hZ} + \norm{\hZ - \Z} \\
    &\leq \hrf(k) \frac{\Bo}{\pd} 
    + \norm{\hZ - \Z}.
\end{align*}
Finally, we have 
\begin{align*}
\lVert\hZ &- \Z\rVert \\
&\leq \norm{\hA^{-1}\hB - \hA^{-1}\B + \hA^{-1}\B - \A^{-1}\B } \\
 &\leq \norm{\hA^{-1}} \norm{\hB - \B} + \norm{\B} \norm{\hA^{-1} - \A^{-1}}\\
 &\leq \frac{\Low\rf(t)\D}{\pd} 
 + \Bo\norm{\hA^{-1} - \A^{-1}}\\
 &\leq \frac{\Low\rf(t)\D}{\pd} + \Bo \norm{\hA^{-1}}\norm{\A - \hA}\norm{\A^{-1}}\\
 & \leq \frac{\Low\rf(t)\D}{\pd} + \frac{\Bo\rhow \rf(t)\D}{\pd^2}.
\end{align*}
Combining all together we get \eqref{eq:boundInvertible}.
As regards the second part of the statement, if Assumption~\ref{ass:contraction} is satisfied, then, in view of  Lemma~\ref{lm:optnormcontractive}, we can take  $\pd = 1-\q$ in \eqref{eq20200123c} and obtain \eqref{eq:boundInvertibleContraction}.
\end{proof}

The following two propositions allow us to derive the refined iteration complexity bound for AID-FP.

\begin{proposition}
\label{thm:20200124a}
Suppose that \eqref{ass:fixedpointmethod} holds.
Let $\lambda \in \Lambda, t \in \N$. 
Let $\wpo = 0 \in \R^{d\times n}$ and for every integer $k \geq 1$,
\begin{equation*}
    \wpk  =  \jac_1\Phi(\wt, \lambda) \wpkm + \jac_2\Phi(\wt, \lambda).
\end{equation*}
Then, for every $k \in \N$,
\begin{equation}
	 \wpk^{\!\top}\grad_1 \fo (\wt, \lambda)  = \jac_2\Phi (\wt, \lambda)^\top \!\vk.
\end{equation}
\end{proposition}
\begin{proof}
We set $Y =\partial_1 \Phi(\wt,\lambda) \in \R^{d\times d}$, $C = \jac_2\Phi(\wt, \lambda) \in \R^{d\times n}$,
and $b = \nabla_1 E(\wt, \lambda) \in \R^d$. Let $k \in \N$, $k \geq 1$. Then, 
 \begin{align*}
	\wpk &= Y\wpkm  + C \\
	&= Y^2\wpkmm  + (1 + Y)C \\ 
	&\;\;\vdots \notag \\[-0.5ex]
	&=   Y^k \wpo + \sum_{i=0}^{k-1} Y^i C\\[-0.5ex]
	&= \sum_{i=0}^{k-1} Y^i  C.
\end{align*}
In the same way, it follows 
from \eqref{ass:fixedpointmethod} that
$\vk = Y^\top v_{t,k-1}(\lambda) + b =
\sum_{i=0}^{k-1} (Y^\top)^i b$.
Therefore,  we have
\begin{align*}
	 \wpk^\top b &= C^\top \left(\sum_{i=0}^{k-1} Y^i \right)^\top b \\
	 &= C^\top \sum_{s=i}^{k-1} (Y^\top)^i b\\
	 &= C^\top \vk
\end{align*}
and the statement follows.
\end{proof}

Using Proposition~\ref{thm:20200124a}, for AID-FP we can write
\begin{equation}
\label{eq:20200124d}
  \hat\grad{f}(\lambda) \!=\! 
  \grad_{\!2} \fo (\wt, \lambda) +
  \wpk^{\!\top} \!\grad_{\!1}\fo (\wt, \lambda).
\end{equation}
Then a result similar to Proposition~\ref{lm:boundDiff}
can be derived.

\begin{proposition}
\label{lm:diffFix} 
Suppose that Assumption~\ref{assA}\ref{ass:fixedpoint}\ref{ass:philip} and Assumption~\ref{ass:contraction} hold. Let $\lambda \in \Lambda$ and $(\wpk)_{k \in \N}$ be defined as in Proposition~\ref{thm:20200124a}. 
Then, for every $t, k \in \N$, with $t\geq 1$,
\begin{multline*}
 \norm{\wpk - w^\prime(\lambda)} \\
 \leq \left(\!\rhol + \rhow\frac{\LPhi}{(1 - \q)}\!\right)\!\frac{\D(1- \q^k)}{1-\q}  \rf(t)
+ \frac{\LPhi}{1 - \q} \q^{k}.   
\end{multline*}
\end{proposition}
\begin{proof}
Let $t, k \in \N$, with $t,k \geq 1$. Recalling that
\begin{align*}
    &\wpk =   \jac_1\Phi( \wt, \lambda) \wpkm  + \jac_2\Phi(\wt, \lambda) \\
     &w^\prime(\lambda)=\jac_1\Phi(w(\lambda), \lambda) w^\prime(\lambda) + \jac_2\Phi(w(\lambda), \lambda)
\end{align*}
we can bound the norm of the difference as follows
\begin{align*}
    \lVert \wpk &- w^\prime(\lambda) \rVert\\
    \leq &  \norm{\jac_1\Phi ( \wt, \lambda) - \jac_1\Phi ( w(\lambda), \lambda)} \norm{w^\prime(\lambda)} \\ 
    &+ \norm{\jac_1\Phi ( \wt, \lambda)} \norm{\wpkm - w^\prime(\lambda)} \\
    & +\norm{\jac_2\Phi ( \wt, \lambda) - \jac_2\Phi ( w(\lambda), \lambda)}\\
    \leq &  (\rhol + \rhow\LPhi/(1-\q)) 
    \norm{\wt - w(\lambda)} \\
    &+ \q \norm{\wpkm - w^\prime(\lambda)},
\end{align*}
which gives a recursive inequality. Then, setting $p :=  \rhol + \rhow\LPhi/(1-\q)$, $\Delta_{t} : 
= \norm{\wt - w(\lambda)}$ and $\hat{\Delta}^\prime_k := \norm{\wpk - w^\prime(\lambda)}$, we have 
\begin{equation*}
    \hat{\Delta}^\prime_k \leq 
    \q\hat{\Delta}^\prime_{k-1} + p \Delta_{t}.
\end{equation*}
Therefore, it follows from 
Lemma~\ref{lem:sequence}\ref{eq:20200124f} with $\tau =p \Delta_{t}$, that
\begin{align*}
    \hat{\Delta}^\prime_k 
    &\leq \q^{k}\hat{\Delta}^\prime_0 + p \Delta_{t} \frac{1 - \q^k}{1-\q} \\
    &\leq \frac{\LPhi}{1 - \q} \q^{k} 
    + p \D\rf(t)\frac{ 1- \q^k}{1-\q},
\end{align*}
where in the last inequality we used
Assumption~\ref{assC}\ref{ass:innerconv} and 
(see \eqref{eq:20200124c})
$\hat{\Delta}^\prime_0 = \norm{\wpo - w^\prime(\lambda) } = \norm{w^\prime(\lambda)}  
   \leq \LPhi/(1 - \q)$. 
The statement follows.
\end{proof}

\boundFixed*
\begin{proof}
Let $t \in \N$ with $t\geq 1$ and let
 $(\wpk)_{k \in \N}$ be defined as in Proposition~\ref{thm:20200124a}. Then,
the difference between exact and approximate gradients
can be bound as follows
\begin{align*}
    \lVert \hat\grad f(\lambda) &- \grad f(\lambda)\rVert\\  
    \leq & \norm{\grad_2 \fo(\wt, \lambda) - \grad_2 \fo(w(\lambda), \lambda) } \\
    &+\norm{w^\prime(\lambda)} 
    \norm{\grad_1 \fo(\wt, \lambda) - \grad_1 \fo(w(\lambda),\lambda)} \\
    &+ \norm{w^\prime(\lambda) - \wpk} \norm{\grad_1 \fo(\wt, \lambda)}.
\end{align*}
Now note that
$\norm{w_t(\lambda)} \leq \norm{w_t(\lambda) - w(\lambda)} + \norm{w(\lambda)} \leq (\rf(t) + 1) \norm{w(\lambda)} \leq 2 \D$.
Then it follows from the assumptions and 
Lemmas~\ref{lm:optnormcontractive} and \ref{lm:boundEPhi} that
\begin{align*}
    \norm{\grad f (\lambda) - \hat\grad{f} (\lambda)} \leq& \left(\Lol + \frac{\Low \LPhi}{1 - \q}\right) \rf(t) \D \\
    &+ \Bo \norm{\wpk - w^\prime(\lambda)},
\end{align*}
and the last term can be bounded using Proposition~\ref{lm:diffFix}.
\end{proof}

\section{Gradient Descent as a Contraction Map}\label{app:gd}

\label{appB}
Consider problem~\eqref{mainprob1} and take
\begin{equation*}
  \Phi(w, \lambda) = w - \sz \grad_1 \fin(w, \lambda),
\end{equation*}
where $\ell\colon \R^d\times \Lambda \to \R$
is twice continuously differentiable and, for every $\lambda \in \Lambda$, 
\vspace{-2ex}
\setlist[enumerate]{itemsep=0mm}
\begin{enumerate}[{\rm(i)}]
    \item $\ell(\cdot, \lambda)$ is 
$\mufi$-strongly convex 
and $\Llw$-Lipschitz smooth, with $\mufi>0$ and $\Llw>0$.
    \item $\alpha\colon \Lambda \subset \R^n \to \R_{++}$ 
is differentiable.
\end{enumerate}
\vspace{-2ex}
Then, if $\sz \in (0, 2/\Llw)$, $\Phi(\cdot, \lambda)$
 is a contraction with constant $\q = \max\{1-\sz\mufi, \sz \Llw - 1\}$. The optimal choice of the step-size leads to set $\sz = 2/(\Llw + \mufi)$ giving
\begin{equation*}
    \q = \frac{\Llw - \mufi}{\Llw + \mufi} = \frac{\cond -1}{\cond +1},
\end{equation*}
where $\cond = \Llw/\mufi$ is the condition number of the lower level problem in \eqref{mainprob1}. 
Note that, for every $t \in \N$ and $\lambda \in \Lambda$,
\begin{equation}
    \mufi I \preccurlyeq \nabla_1^2 \ell(\wt,\lambda) 
    \preccurlyeq \Llw I
\end{equation}
hence the condition number of $\nabla_1^2 \ell(\wt,\lambda)$
is smaller than $\cond$.

We can write the derivatives of $\Phi$ as:
\begin{align}
    \jac_2\Phi(w, \lambda) &= -\grad_1 \fin (w, \lambda) \dsz^\top - \sz \hess_{21} \fin (w, \lambda) \label{eq:jacphi2}\\
    \jac_1\Phi(w, \lambda) &= I - \sz \hess_{1} \fin (w, \lambda)
\end{align}
\begin{remark}\label{re:avoidcrazystuff}
When evaluated in $(w(\lambda), \lambda)$, one does not need $\dsz$ for~\eqref{eq:jacphi2}, because the first term on the r.h.s of \cref{eq:jacphi2} is $0$:
\begin{equation*}
    \jac_2\Phi(w(\lambda), \lambda) = - \sz \hess_{21} \fin (w(\lambda), \lambda).
\end{equation*}
\end{remark}
From the remark it follows that $\norm{\jac_2\Phi(w(\lambda), \lambda)} = \sz \norm{\hess_{21} \fin (w(\lambda), \lambda) }$

Furthermore, if we assume $\hess_{1}\fin (\cdot, \lambda)$ is $\rhowl$-Lipschitz and $\hess_{21}\fin (\cdot, \lambda)$ is $\rholl$-Lipschitz then, calling $\Delta_{\jac_1 \Phi} : = \norm{\jac_1\Phi(w_1, \lambda) - \jac_1\Phi(w_2, \lambda)}$ and $\Delta_{\jac_2 \Phi} : = \norm{\jac_2\Phi(w_1, \lambda) - \jac_2\Phi(w_2, \lambda)}$ we have:
\begin{align*}
   \Delta_{\jac_1 \Phi} &= \norm{\sz \left(\hess_{1} \fin (w_1, \lambda) - \hess_{1} \fin (w_2, \lambda) \right)} \\
    &\leq \sz \rhowl \norm{w_1 - w_2}.
\end{align*}
and
\begin{alignat*}{2}
    \Delta_{\jac_2 \Phi} &&=& \lVert \left(\grad_1 \fin (w_1, \lambda) - \grad_1 \fin (w_2, \lambda)\right)\dsz^\top \\
    && & + \sz \left(\hess_{21} \fin (w_1, \lambda) - \hess_{21} \fin (w_2, \lambda) \right)\rVert \\
    && \leq & (\Llw \norm{\dsz} + \sz\rholl) \norm{w_1 - w_2}.
\end{alignat*}
Thus, Assumption~\ref{assA}\ref{ass:philip} holds with $\rhow = \sz \rhowl$ and 
$\rhol = \Llw \norm{\dsz} + \sz\rholl$.
Moreover,
if we pick $L_{\fin,\lambda}$ such that $\norm{\hess_{21} \fin (w(\lambda), \lambda) } \leq L_{\fin,\lambda}$, then
Theorems~\ref{boundSolver},\ref{th:gradBoundImplicit} and \ref{boundFixed}
hold with 
\begin{align*}
    \q &= \max\{1-\sz\mufi, \sz \Llw - 1\} \\
    c_1(\lambda) &:= \left(\Lol  + \frac{\Low \sz L_{\fin,\lambda}}{1-\q}\right) \D \\
    c_2(\lambda) &:= \left(\Llw \norm{\dsz} + \sz\rholl \right) \Bo \,\D \\  &  \ \ + \frac{\rhowl \sz^2 L_{\fin,\lambda} \Bo \,\D}{1 - \q} \\
    c_3 (\lambda) &:= \frac{\Bo \,\sz \norm{\hess_{21} \fin (w(\lambda), \lambda) }}{(1-\q)}.
\end{align*}

Given \Cref{re:avoidcrazystuff} and to avoid additional complexity of the algorithm, we can consider replacing $\jac_2\Phi(w, \lambda)$ with $\hat \jac_2\Phi(w, \lambda) = - \sz \hess_{21} \fin (w, \lambda)$ in the expression for both $\grad f_t(\lambda)$ and  $\hat \grad f(\lambda)$.
We apply this change in all the experiments of case~\eqref{mainprob1}.


\section{Experiments}\label{app:exp}

\subsection{Hypergradient Approximation}\label{sec:datagen}
In this section we provide details for the experiments in \Cref{sec:4.1}.

We define the train and validation kernel matrices in~\eqref{eq:krr} as follows: 
\begin{align*}
    \kerv(\gamma)_{i,j} &= \exp \left[- \left(\Xv_i -\Xt_j\right)^\top {\rm diag}(\gamma) \left(\Xv_i -\Xt_j\right)\right] \\
    \kert(\gamma)_{i,j} &= \exp\left[-\left(\Xt_i -\Xt_j\right)^\top {\rm diag}(\gamma) \left(\Xt_i -\Xt_j\right)\right].
\end{align*}

We generate synthetic data by sampling each element of $\Xt$ and $\Xv$ from a normal distribution.  $\yt$ (and in the same way $\yv$)  is subsequently obtained in the following ways for the different settings outlined in \Cref{sec:4.1}.
\begin{alignat*}{2}
    \yt &= sign( \Xt w_* + m\epsilon) \quad \qquad &&\text{(LR)}\\
    \yt &= \Xt w_* + m\epsilon \quad \qquad &&\text{(KRR)} \\
    \yt &= \Xt(w_* + b_*) + m\epsilon   \qquad &&\text{(BR)}\\
    \yt &= \Xt H_* w_* + m\epsilon  \qquad &&\text{(HR)}
\end{alignat*}
where $sign$ is the elementwise sign function, each element of $\epsilon$, $w_*$  and $H_*$ is sampled from a normal distribution, $b_*= \mathbb{1}$\footnote{where $\mathbb{1} \in \R^d$ is a vector with all its components set to one.}, and $m=0.1$. $\Xt$,$\Xv$ have dimension $50 \times 100$  while $H$ is a $100 \times 200$ matrix. 
The results in \Cref{fig:synthlambdas} report mean and std over  $20$ values of $\lambda$ such that $ \lambda_i \sim \U(\lambda_{\min}, \lambda_{\max})$ for $1 \leq i \leq n$ where $\U$ is the uniform distribution on the interval [$\lambda_{\min}$, $\lambda_{\max}$] which is  $[0.01, 10]$, $[0.0005, 0.005]$, $[-5, 5]$ and $[-1, 1]$ respectively for LR, KRR, BR, HR. Furthemrmore we set $\beta=1$ for BR and  $\beta=10$ for HR. $\lambda_{\min}$, $\lambda_{\max}$ and $\beta$ are selected as to make the expected lower-level problem difficult ($\q$ close to $1$). 

We note that in \Cref{fig:synthlambdas} the asymptotic error for KRR, BR and HR is considerably large. We suspect that this is due to the numerical error made by the hypergradient approximation procedures being larger than the one made when computing the exact hypergradient using the closed form expression of $w(\lambda)$.
Indeed, we have observed that using double precision halves the asymptotic error, but we did not investigate further. Our theoretical analysis does not take this source of error into account since it assumes infinite precision arithmetic.

\subsection{Bilevel Optimization}\label{app:bopt}
This section contains the details and some additional results on the experiments in \Cref{sec:4.2} on problems of type~\eqref{mainprob1}\footnote{This includes all the settings except equilibrium models.}.

The average cross-entropy in 20 newsgroup and Fashion MNIST is defined as
\begin{equation*}
\mathrm{CE}(Z, y) = -\frac{1}{|y|}\sum_{k=1}^{|y|}\sum_{i=1}^c \delta_{i, y_{k}} \log\left(\frac{e^{Z_{ki}}}{\sum_{j=1}^c e^{Z_{kj}}} \right) 
\end{equation*}
where $Z_k \in \R^{c}$ $y_k \in \{0,\dots, c\}$ are respectively the prediction scores and the class label for the $k$-th example, $\delta_{i, y_{k}}$ equals $1$ when $i=y_{k}$ and $0$ otherwise and $|y|$ is the number of examples.

To solve the upper-level problem we use gradient descent with fixed step-size where the gradient is estimated using \ITD{} or \IMD{} methods. In particular, we generate the sequence $(\lambda_i)^s_{i}$ as follows:
\begin{equation*}
    \lambda_i = \lambda_{i-1} -\zeta g(\lambda_{i-1}) 
\end{equation*}
where $g(\lambda)= \nabla f_t(\lambda)$ for \ITD{} and $g(\lambda) = \hat \nabla f(\lambda)$ for \IMD{} are computed respectively using  \Cref{algo1} and \Cref{algo2} with $t$ and $k$ fixed throughout the optimization. 

All methods compute $w_t(\lambda)$ using $t$-steps of the same algorithm solving the lower-level problem in~\eqref{mainprob1}. In particular 
we use heavy ball with optimal constants for Parkinson and gradient descent with step-size manually chosen for the other two settings where it is harder to compute the optimal one. Specifically, we set the step-size to $10^3$ for 20 newsgroup and $10^4$ for Fashion MNSIT\footnote{Note that in this case the step-size is constant w.r.t. $\lambda$ whereas the optimal one would vary with $\lambda$.}.

The initial parameter $\lambda_0$ is set to $(\beta_0, \gamma_0) = (0, - \log(p) \mathbb{1})$\footnote{where $\mathbb{1} \in \R^p$ is a vector with all its components set to one.} for Parkinson, $0 \in \R^p$ for 20 newsgroup and $X_0 = 0 \in \R^{c \times p}$ for Fashion MNIST. Furthermore, the regularization parameter $\beta$ is set to $1$ for Fashion MNIST.

We choose the step-size $\zeta$ with a grid search over 30 values in a suitable interval for each problem, choosing the one bringing the lowest value of the approximate objective $f_t(\lambda_s) = E(w_t(\lambda_s), \lambda_s)$ where $s$ is equal to $1000$, $500$, $4000$ for Parkinson, 20 newsgroup and Fashion MNIST respectively. The grid search values are spaced evenly in log scale inside the intervals $[10^{-6}, 10]$,  $[10^{-4}, 10^4]$ and  $[10^{-10}, 10^{-2}]$ respectively for Parkinson, 20 newsgroup and Fashion MNIST.

We note that the results In Table~\ref{tb:objective} report the value of the approximate objectve $f_t(\lambda_s) = \fo(w_t(\lambda_s), \lambda_s)$ and the test accuracy (computed on $w_t(\lambda_s)$). For completeness, in Table~\ref{tb:trueobj} we report $f(\lambda_s) = \fo (w(\lambda_s), \lambda_s))$ and the test accuracy (computed on $w(\lambda_s)$) where $w(\lambda_s)$ is computed using RMAD (exploiting the closed form of $w(\lambda_s)$) for Parkinson and using  $2000$ steps of gradient descent starting from $w_0(\lambda)=0$ for 20 newsgroup and Fashion MNIST.

\begin{table*}[th]
\small
\caption{The values of $f(\lambda_s)$ and test accuracy (in percentage) are displayed after $s$ gradient descent steps, where $s$ is $1000$, $500$ and $4000$  for Parkinson, 20 news and Fashion MNIST respectively. $k_r=10$ for Parkinson and 20 news while for Fashion MNIST $k_r=5$.}\label{tb:trueobj}
\begin{tabular}[t]{lrr}
\multicolumn{3}{c}{\textbf{Parkinson}} \\
\toprule
&   $t=100$ &   $t=150$ \\
\midrule
 ITD       & 2.39 (75.8) & 2.11 (69.7) \\
 FP $k=t$  & 2.36 (81.8) & 2.19 (77.3) \\
 CG $k=t$  & 2.20 (78.8) & 2.19 (77.3) \\
 FP $k=k_r$ & 2.71 (80.3) & 2.60 (78.8) \\
 CG $k=k_r$ & 2.17 (78.8) & 1.99 (77.3) \\
\bottomrule
\end{tabular}
\hfill
\begin{tabular}[t]{rrr}
\multicolumn{3}{c}{\textbf{20 newsgroup}} \\
\toprule
$t=10$ &   $t=25$ &   $t=50$ \\
\midrule
 1.155 (59.4) & 1.082 (61.1) & 1.058 (61.6) \\
 1.155 (59.5) & 1.083 (61.1) & 1.058 (61.6) \\
 0.983 (62.9) & 0.955 (62.9) & 0.946 (63.5) \\
 1.155 (59.5) & 1.078 (61.7) & 1.160 (59.1) \\
 0.983 (62.9) & 0.989 (62.6) & 1.001 (62.3) \\
\bottomrule
\end{tabular}
\hfill
\begin{tabular}[t]{rrr}
\multicolumn{2}{c}{\textbf{Fashion MNIST}} \\
\toprule
$t=5$ &   $t=10$ \\
\midrule
 0.497 (84.1) & 0.431 (83.8) \\ 
 0.497 (84.1) & 0.431 (83.8) \\ 
 0.522 (83.8) & 0.424 (84.0) \\ 
 0.497 (84.1) & 0.426 (83.9) \\ 
 0.522 (83.8) & 0.424 (84.0) \\ 
\bottomrule
\end{tabular}
\end{table*}

\subsection{Equilibrium Models with Convolutions} \label{sec:apx:eqmconv}
\begin{figure*}[t]
\vspace{-.1truecm}
    \centering
    \includegraphics[width=\textwidth]{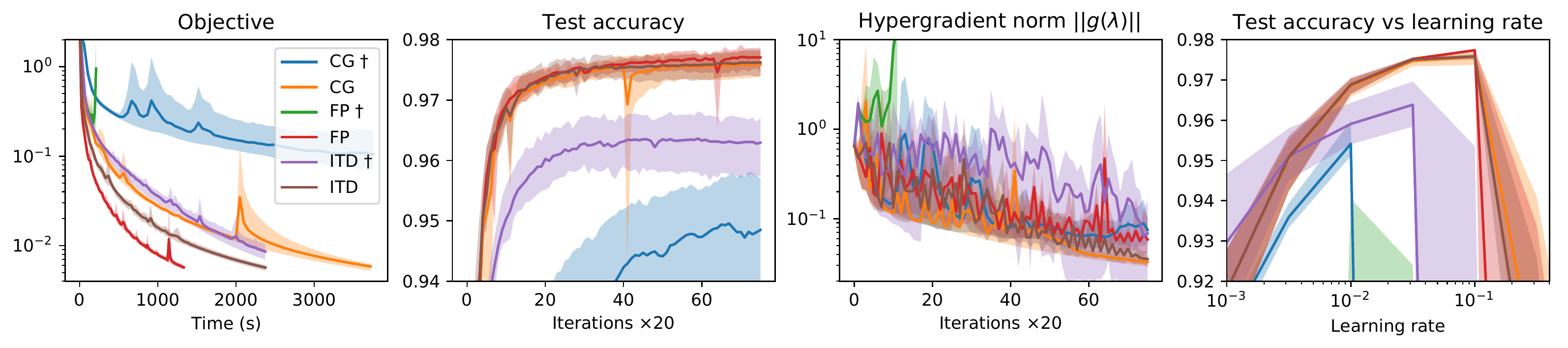}
    \vspace{-5mm}
    \caption{Experiments with convolutional EQMs.
    Mean (solid line) and point-wise minimum-maximum range (shaded region) across 5 random seeds. The seed only controls the initialization of $\lambda$.
    The estimated hypergradient $g(\lambda)$ is 
    equal to $\nabla f_t(\lambda)$ for \ITD{} and $\hat{\nabla} f(\lambda)$ for \IMD.
    We used $t=k=20$ for all methods and Nesterov momentum (1500 iterations) for optimizing $\lambda$, applying a projector operator at each iteration except for the methods marked with $\dag$. Note that in the first three plots the step-size for the unconstrained experiments is smaller, to prevent divergence.
    }
    \label{fig:conv_eqm}
\end{figure*}
\begin{figure*}[t]
    \centering
    \includegraphics[width=\textwidth]{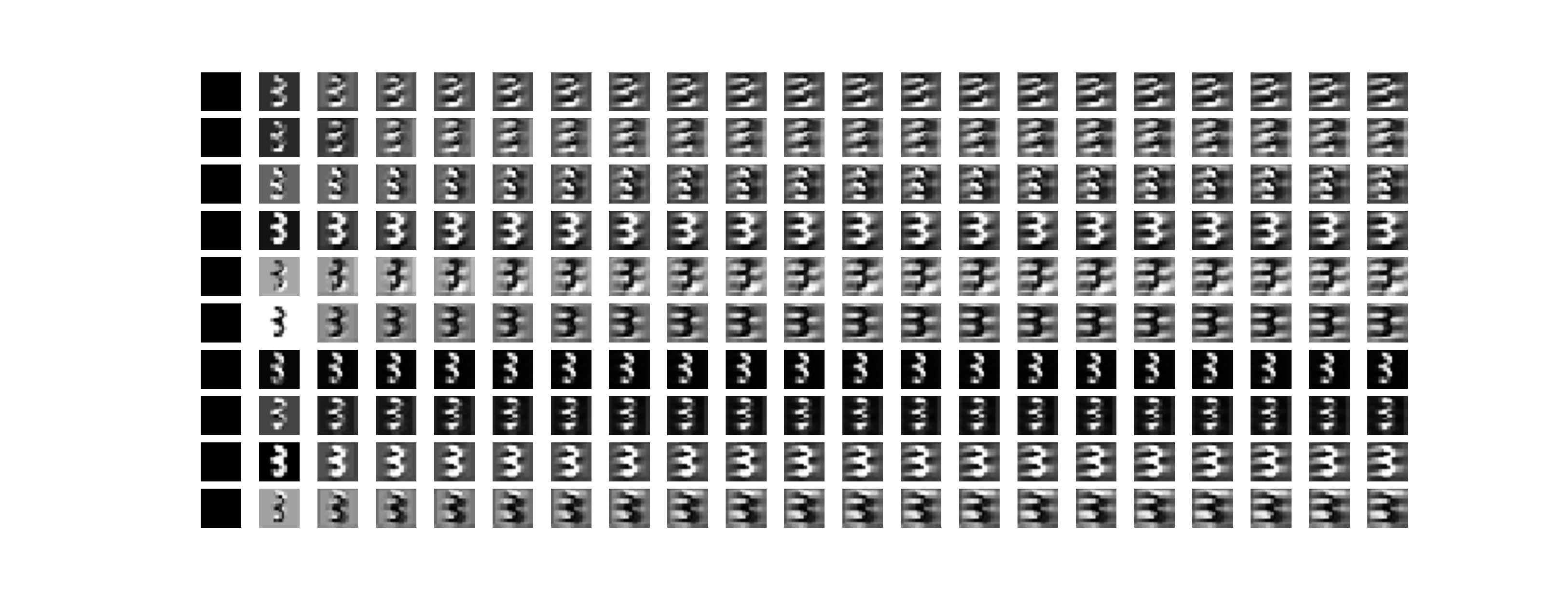}
    \includegraphics[width=\textwidth]{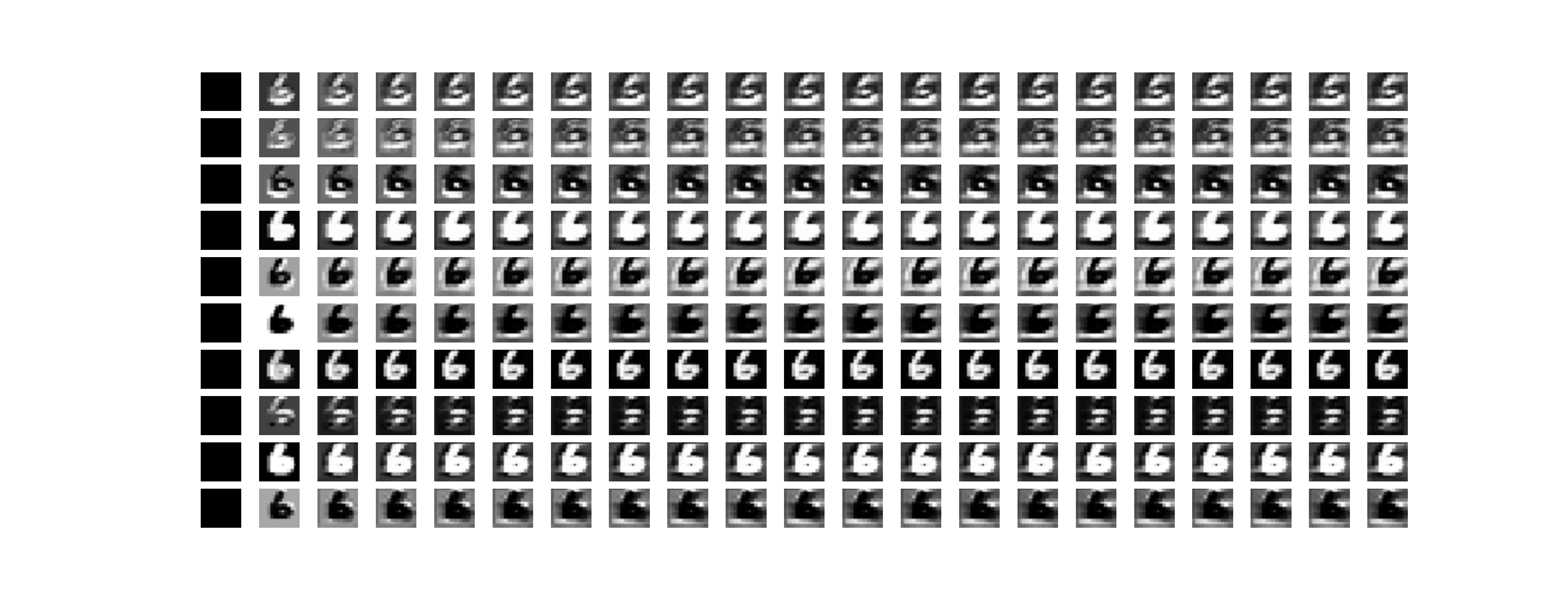}
    \caption{Images of two samples of the states filter maps $w_i \in \R^{10\times 14\times14}$ for a three and a six from the MNIST dataset, learnt with the fixed-point method and with projection. Each of the ten rows represents a filter and the x-axis proceeds with the iterations of the EQM dynamics (for a total of $t=20$ iterations). The states are initialized to $0$ (black images on the left) and then the mapping \eqref{eq:EQM_CONV} is iterated 20 times to approximately reach the fixed point representation (rightmost images).}
    \label{fig:states}
\end{figure*}
In this section we report a series of experiments on equilibrium models quite similar to those of the last paragraph of Section \ref{sec:4.2}, but with convolutional and max-pooling operators in place of the affinities of Equation \eqref{eq:qlin_eqm}.  In particular we model the learnable dynamics with parameters $\gamma=(K, K', c)$ as
\begin{equation}
    \label{eq:EQM_CONV}
    \phi_i(w_i, \gamma) = \tanh{\left(K\star w_i + \mu_{2\times 2}(K' \star x_i) + c
    \right)}
\end{equation}
where $w_i\in\R^{h\times 14 \times 14}$ are the state feature maps, $\star$ denotes multi-channel bidimensional cross-correlation, $K$ and $K'$  contain $h$ $3\times 3$ convolutional kernels each
and $\mu_{2\times 2}$ denotes the max-pooling operator with a $2\times 2$ field and stride of $2$.   
The state feature maps are passed through a max-pooling operator before being flattened and fed to a multiclass logistic classifier. We set $h=10$ for all the experiments. 
We use the results and the code of \citet{sedghi2019singular}
to efficiently perform the projection of the linear operator associated to $K$ into the unit spectral ball\footnote{Specifically, we project onto $\norm{\mathrm{c}(K)} \leq 0.999$, where $\mathrm{c}(K)$ is an $h\times h$ matrix of doubly block circulant matrices; see \citet{sedghi2019singular} for details.}. 
Data and optimization method for the upper objective are the same of Section \ref{sec:4.2}.

The results, reported in Figure \ref{fig:conv_eqm}, show similar behaviours of those in Section \ref{sec:4.2}, albeit with more marked differences among the methods, especially for the experiments without projection (denoted by $\dag$
 in the figure). The statistical performances of the contractive convolutional EQM exceed abundantly those given by simpler dynamics of \eqref{eq:qlin_eqm}, with the fixed-point method (red line) being slightly better then the others. We show some visual examples of the learned dynamics in Figure \ref{fig:states}, where we plot the 10 bidimensional state filter maps as the iterations of \eqref{eq:EQM_CONV} proceed. 
 
 Interestingly, when the projection is not performed, optimization with the fixed-point scheme to compute the hypergradient (akin to recurrent backpropagation, see green shaded region in the rightmost plot of Figure \ref{fig:conv_eqm}) 
 does not reliably converge for all the probed values of the step-size,
 indicating once more the importance of the contractiveness assumption for AID methods. 
We finally note that regularizing the norm of $\partial_1\phi_i$ or adding $L_1$ or $L_{\infty}$ penalty terms on the matrix of the state-wise linear transformation may encourage, but does not strictly enforce, such condition. This may in part explain some difficulties previously encountered in training EQM-like  models, e.g. in the context of relational learning (graph neural networks).

\end{document}